\documentclass{article}

\PassOptionsToPackage{numbers, compress}{natbib}

\usepackage[preprint]{neurips_2026}

\usepackage[utf8]{inputenc} 
\usepackage[T1]{fontenc}    
\usepackage{hyperref}       
\usepackage{url}            
\usepackage{booktabs}       
\usepackage{amsfonts}       
\usepackage{nicefrac}       
\usepackage{microtype}      
\usepackage{xcolor}         

\usepackage{booktabs}
\usepackage{algorithm}
\usepackage{algpseudocode}
\usepackage{amsthm}
\usepackage{mathtools}
\usepackage{amssymb}
\usepackage{cleveref}
\usepackage[colorinlistoftodos, prependcaption, textsize=tiny]{todonotes}
\usepackage{amsmath, amssymb, amsthm}
\usepackage{graphicx}
\usepackage{enumitem}
\usepackage{subcaption}
\usepackage{wrapfig}
\usepackage[acronym]{glossaries}
\glsdisablehyper

\theoremstyle{remark}

\theoremstyle{plain} 
\newtheorem{theorem}{Theorem}
\newtheorem{proposition}{Proposition}

\title{CIG: Exploration via Conditional Information Gain}

\author{%
  Tim Joseph \\
  FZI\\
  Karlsruhe, Germany \\
  \texttt{joseph@fzi.de} \\
  \And
   Marcus Fechner \\
   KIT  \\
   Karlsruhe, Germany \\
  \texttt{marcus.fechner@kit.edu} \\
   \AND
   Philipp Stegmaier \\
   KIT  \\
   Karlsruhe, Germany \\
   \And
   Karam Daaboul \\
   FZI\\
   Karlsruhe, Germany \\
   \And
   J. Marius Zöllner \\
   KIT  \\
   Karlsruhe, Germany \\
}

\begin{document}
\newacronym{cig}{CIG}{Conditional Information Gain}

\maketitle

\begin{abstract}
Intrinsic rewards for exploration in reinforcement learning condition on different contexts: lifelong rewards score each transition against accumulated experience but ignore within-rollout redundancy; episodic rewards penalize intra-trajectory repetition but discard lifetime progress. Hybrid methods combine both signals through heuristic weights or require Gaussian-process dynamics that do not scale beyond low-dimensional state spaces. Trajectory-level information gain decomposes into per-step terms that condition on the replay buffer and rollout prefix simultaneously, but remains intractable for deep models. We derive the \textbf{C}onditional \textbf{I}nformation \textbf{G}ain (CIG) reward as a tractable surrogate: a log-determinant objective over an ensemble disagreement kernel whose Cholesky factorization yields causal per-step rewards that retain both conditioning sets while scaling to high-dimensional state spaces. We instantiate CIG in a model-based setting, where rollouts are short and within-rollout corrections remain largely unexplored. Across twelve tasks spanning discrete (MiniGrid) and continuous control (OGBench), in both clean and stochastic-distractor settings, CIG outperforms or matches prior exploration methods while remaining robust to stochastic distractors.
\end{abstract}

\section{Introduction}
\label{sec:introduction}

Reinforcement learning agents in environments with sparse or delayed rewards must explore: they must visit states whose outcomes are not yet predictable from past experience. Intrinsic rewards formalize this pressure by assigning a bonus to transitions that are informative about the environment~\cite{schmidhuber_formal_1991, 10.3389/neuro.12.006.2007}, whether combined with extrinsic reward during task-directed training~\cite{pathakCuriosityDrivenExplorationSelfSupervised2017, burdaExplorationRandomNetwork2019} or used as the sole objective during unsupervised pre-training of a world model~\citep{sekarPlanningExploreSelfSupervised2020, hafnerMasteringDiverseControl2025}. The intrinsic reward controls which transitions the agent collects and therefore which dynamics the world model learns.

In model-based RL the policy is optimised over short imagined rollouts~\cite{haRecurrentWorldModels2018a, hafnerLearningLatentDynamics2019, hafnerMasteringAtariDiscrete2021}, and two sources of context are available for scoring a transition within such a rollout (Figure~\ref{fig:conditioning}). \emph{Lifelong} rewards condition on the replay buffer~\citep{pathakCuriosityDrivenExplorationSelfSupervised2017, burdaExplorationRandomNetwork2019, pathakSelfSupervisedExplorationDisagreement2019, sekarPlanningExploreSelfSupervised2020}: they assign full credit to every transition that probes a model gap, regardless of whether the current rollout has already probed it, so an imagined trajectory in a maze can spend every step revisiting the same uncertain corridor junction while neglecting independent frontiers elsewhere. \emph{Episodic} rewards condition on the rollout prefix~\citep{henaffExplorationEllipticalEpisodic2022, henaffStudyGlobalEpisodic2023}: they penalise intra-trajectory repetition but treat a familiar region and a genuinely uncertain one identically whenever both are new within the rollout. These methods assume long real episodes that provide a rich prefix for the redundancy signal; we are not aware of an extension to the short imagined rollouts of model-based RL. Existing hybrids combine the two signals through heuristic weights~\citep{badiaNeverGiveLearning2020, zhangNovelDSimpleEffective2021} or require Gaussian-process dynamics restricted to low-dimensional state spaces~\citep{mehtaExplorationPlanningInformation2022}.

Trajectory-level information gain provides an alternative that conditions on both the buffer and the prefix simultaneously: its chain-rule decomposition yields per-step rewards in which each term is conditioned on the replay buffer through the trained ensemble and on the rollout prefix through the preceding steps~\citep{10.1214/aoms/1177728069}. Computing these terms, however, requires marginalising over a posterior on model parameters that is intractable for deep networks.

We derive the \textbf{Conditional Information Gain (CIG)} reward, a tractable surrogate for this objective. When the posterior over model parameters is approximated by a finite ensemble of one-step predictors, the trajectory-level information gain reduces to a log-determinant over a kernel of pairwise ensemble disagreements (Eq.~\ref{eq:gram_matrix}). We identify a rank-saturation bottleneck of the na\"ive surrogate and resolve it through a trace reduction that raises the effective capacity from $M{-}1$ to $\min(T,\,(M{-}1)d)$. The Cholesky factorisation of this log-determinant then yields causal per-step rewards: each step's reward equals its total disagreement minus the portion already explained by preceding steps (Eq.~\ref{eq:cig_reward}). CIG plugs into an existing model-based backbone without additional hyperparameters; the only new quantity, the aleatoric scale, is estimated post-hoc from training residuals. Across twelve tasks spanning discrete (MiniGrid) and continuous control (OGBench), in both clean and stochastic-distractor settings, CIG outperforms or matches all baselines while remaining robust to stochastic distractors.

We contribute the CIG reward itself (\S\ref{sec:method}), a per-step intrinsic reward derived from a log-determinant surrogate of trajectory-level information gain that conditions on both the replay buffer and the rollout prefix, together with a rank-saturation analysis that motivates the trace reduction resolving the capacity bottleneck of the finite-ensemble surrogate (\S\ref{sec:surrogate}). We evaluate CIG on twelve tasks spanning discrete and continuous-control domains, in both clean and stochastic-distractor variants, using a reward-free protocol that isolates exploration quality from extrinsic reward exploitation; CIG is the only method in the top tier on every task, with an aggregate normalised IQM exceeding the next-best baseline by a clear margin (\S\ref{sec:experiments}).
\section{Background}
\label{sec:background}
\subsection{Setup}
\label{sec:bg_setup}
We consider an agent interacting with a reward-free Markov decision process $\mathcal{M} = (\mathcal{S}, \mathcal{A}, P, \gamma)$, where $\mathcal{S} \subseteq \mathbb{R}^d$ is the state space, $\mathcal{A}$ the action space, $P(s' \mid s, a)$ the unknown transition distribution, and $\gamma \in [0, 1)$ a discount factor. At each step the agent samples $a_t \sim \pi_\theta(\cdot \mid s_t)$, observes $s_{t+1} \sim P(\cdot \mid s_t, a_t)$, and appends the transition to a replay buffer~$\mathcal{D}$. The agent's objective is pure exploration: in the absence of any extrinsic reward, it must collect transitions that are maximally informative about the unknown dynamics~$P$. To this end it constructs an intrinsic reward~$r_t$ and optimises $J(\pi_\theta) = \mathbb{E}_{\pi_\theta}\!\bigl[\sum_{t \geq 0} \gamma^{t} r_t\bigr]$.

\subsection{Intrinsic Reward Design and Prior Work}
\label{sec:bg_proxies}
\begin{figure}
    \centering
    \includegraphics[width=\linewidth]{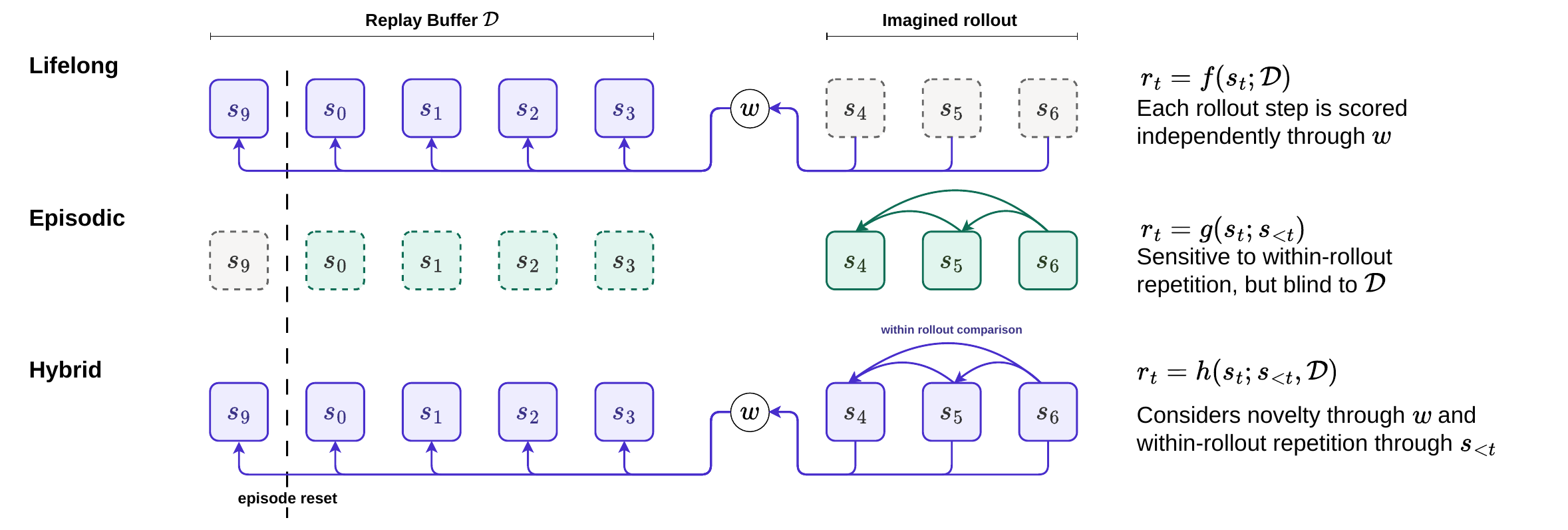}
    \caption{\textbf{Conditioning contexts of intrinsic rewards in model-based RL.} The policy plans over short imagined rollouts. Neural network weights $w$, trained on the replay buffer $\mathcal{D}$, summarize lifetime uncertainty. Solid states and arrows mark the contexts each reward class conditions on; dashed/faded elements are not used as context. In the model-based setting, the rollout prefix $s_{<t}$ is too short to supply a rich episodic signal. CIG (\S\ref{sec:method}) instantiates the hybrid form.}
    \label{fig:conditioning}
\end{figure}

The reward $r_t$ must quantify how informative a transition is, yet informativeness depends on context. Two sources of context are available at step~$t$: the replay buffer~$\mathcal{D}$, which summarizes lifetime experience and determines which dynamics remain uncertain, and the rollout prefix~$s_{<t}$, which records what the current trajectory has already probed. In model-based RL, where the policy optimizes over short imagined rollouts ($T \approx 15$), the prefix represents what the agent \emph{plans} to probe rather than what it has already visited in a real episode. Prior intrinsic rewards differ in which of these two contexts they condition on (Figure~\ref{fig:conditioning}); we summarize the three resulting categories below and provide a comprehensive treatment of related work in Appendix~\ref{sec:related_work}.

\textbf{Lifelong rewards}, $r_t = f(s_t;\, \mathcal{D})$, score each transition against accumulated experience through prediction error~\citep{pathakCuriosityDrivenExplorationSelfSupervised2017, burdaExplorationRandomNetwork2019}, ensemble disagreement~\citep{pathakSelfSupervisedExplorationDisagreement2019}, or density estimates~\citep{bellemareUnifyingCountBasedExploration2016}. Each step is scored independently of all other steps in the rollout, so correlations between disagreement patterns at different time steps are invisible to the reward: correlated steps receive full cumulative credit, a structural redundancy that conditioning on~$\mathcal{D}$ alone cannot detect.
\textbf{Episodic rewards}, $r_t = g(s_t;\, s_{<t})$, penalize within-trajectory repetition~\citep{savinovEpisodicCuriosityReachability2019, henaffExplorationEllipticalEpisodic2022, joLECOLearnableEpisodic2022a, henaffStudyGlobalEpisodic2023} but condition on the prefix alone, without access to the lifetime uncertainty encoded in~$\mathcal{D}$. These methods were developed for model-free RL, where long real episodes supply a rich prefix for the redundancy signal. In the model-based setting we consider, the prefix is a short imagined rollout and carries far less information, limiting the episodic correction.
\textbf{Hybrid rewards} target the joint form $r_t = h(s_t;\, s_{<t}, \mathcal{D})$. Several methods multiply or gate an episodic indicator with a lifelong bonus~\citep{badiaNeverGiveLearning2020, badiaAgent57OutperformingAtari2020, raileanuRIDERewardingImpactDriven2020, zhangNovelDSimpleEffective2021}, though all set the combination weights heuristically rather than deriving them from a unified objective. Plan2Explore~\citep{sekarPlanningExploreSelfSupervised2020} maximizes per-step ensemble disagreement in a model-based setting but scores each step independently, placing it in the lifelong category by conditioning structure despite operating over planned rollouts. TIP~\citep{mehtaExplorationPlanningInformation2022} derives per-step rewards from the trajectory-level information-gain objective of Eq.~\eqref{eq:chain_rule}, conditioning on both~$\mathcal{D}$ and~$s_{<t}$, but relies on Gaussian-process dynamics and is restricted to low-dimensional state spaces. We share TIP's objective; the next subsection formalises it, and \S\ref{sec:method} derives a tractable surrogate that scales to high-dimensional latent spaces.

\subsection{Information Gain as a Unifying Objective}
\label{sec:bg_infogain}

A principled objective that conditions on both $\mathcal{D}$ and $s_{<t}$ follows from treating pure exploration as information gain about the dynamics model. Let $w$ parameterise a probabilistic model of the unknown dynamics $P$, and let $p(w \mid \mathcal{D})$ denote the posterior after observing $\mathcal{D}$. We seek a policy $\pi^\star$ maximising the trajectory-level information gain $I(w;\, s_{1:T} \mid \mathcal{D})$. This objective decomposes via the chain rule of mutual information into per-step terms that condition on both contexts simultaneously:
\begin{equation}\label{eq:chain_rule}
    I(w;\, s_{1:T} \mid \mathcal{D})
    \;=\; \sum_{t=1}^{T}
    I\bigl(w;\, s_t \mid s_{<t},\, \mathcal{D}\bigr).
\end{equation}
Each summand scores step~$t$ against both the buffer~$\mathcal{D}$ and the prefix~$s_{<t}$, so the two conditioning contexts that \S\ref{sec:bg_proxies} identified as complementary appear as joint requirements of a single objective. This sequential decomposition parallels the chain-rule factorisation of expected information gain in Bayesian optimal experimental design~\citep{10.1214/aoms/1177728069, 10.1214/ss/1177009939, 10.1214/23-STS915}. The obstacle is tractability: $p(w \mid \mathcal{D})$ admits no closed form for deep models, and each per-step term requires marginalising over~$w$ to obtain the predictive distribution over~$s_t$. The next section derives a tractable surrogate that retains both conditioning sets while scaling beyond the low-dimensional settings to which TIP is confined.

\section{Method}
\label{sec:method}

We derive the \textbf{Conditional Information Gain (CIG)} reward as a tractable surrogate for the per-step information gain $I(w;\,s_t \mid s_{<t},\,\mathcal{D})$ identified in \S\ref{sec:bg_infogain}, preserving its conditioning on both the replay buffer and the rollout prefix while scaling to high-dimensional latents. Following Plan2Explore~\citep{sekarPlanningExploreSelfSupervised2020}, we instantiate the dynamics model as an ensemble of one-step predictors over the latent of a learned world model~\citep{haRecurrentWorldModels2018a, hafnerLearningLatentDynamics2019, hafnerMasteringAtariDiscrete2021}; from here on $s_t$ refers to this latent, and all downstream quantities are computed in the latent space. Backbone details are deferred to \S\ref{sec:exp_setup}.

The derivation has three steps. We first replace the intractable posterior with a finite ensemble, reducing trajectory information gain to a marginal-entropy problem (\S\ref{sec:marginal_entropy}). A Gaussian bound on the mixture entropy exposes a rank bottleneck of the finite ensemble, which a trace reduction resolves to yield the CIG surrogate (\S\ref{sec:surrogate}). The Cholesky factorisation of the surrogate then decomposes it into causal per-step rewards (\S\ref{sec:reward_derivation}).

\subsection{From Information Gain to Marginal Entropy}
\label{sec:marginal_entropy}

Two standard approximations make the trajectory information gain tractable. \textbf{A1 (ensemble posterior)} replaces the posterior $p(w \mid \mathcal{D})$ with a uniform distribution over $M$ independently trained ensemble members~\citep{lakshminarayananSimpleScalablePredictive2017, sekarPlanningExploreSelfSupervised2020}, indexed by $k$, so that the buffer conditioning is absorbed into training and dropped from the notation. \textbf{A2 (Markov-Gaussian dynamics)} factorises each member's trajectory into Gaussian one-step transitions with mean $\mu_k(s_t, a_t)$ and isotropic aleatoric covariance $\sigma^2 I_d$ shared across members. Under A2, evaluating each member along a fixed imagined trajectory gives a joint predictive that factorises across steps:
\begin{equation}\label{eq:joint_conditional}
    p(s_{1:T} \mid k) = \mathcal{N}\!\bigl(\boldsymbol{\mu}_k,\, \sigma^2 I_{Td}\bigr), \qquad \boldsymbol{\mu}_k = [\mu_k(s_0, a_0);\, \dots;\, \mu_k(s_{T-1}, a_{T-1})] \in \mathbb{R}^{Td}.
\end{equation}
The conditional entropy $H(s_{1:T} \mid k)$ is therefore a constant in the policy, depending only on the aleatoric variance. Since $I(k;\,s_{1:T}) = H(s_{1:T}) - H(s_{1:T}\mid k)$, maximising information gain reduces to maximising the marginal entropy of a Gaussian mixture:
\begin{equation}\label{eq:reduction_to_marginal}
    \operatorname*{arg\,max}_\pi I(k;\,s_{1:T}) = \operatorname*{arg\,max}_\pi H(s_{1:T}), \qquad p(s_{1:T}) = \frac{1}{M}\sum_{k=1}^{M} \mathcal{N}\!\bigl(\boldsymbol{\mu}_k,\, \sigma^2 I_{Td}\bigr).
\end{equation}

\subsection{Gaussian Surrogate and Capacity Saturation}
\label{sec:surrogate}

The mixture entropy in Eq.~\eqref{eq:reduction_to_marginal} has no closed form. By approximation \textbf{A3 (moment-matched Gaussian)}, we replace the mixture with a Gaussian of matching mean and covariance $\Sigma$ (see Appendix~\ref{app:entropy_bound}). Since the Gaussian maximises entropy among distributions with covariance $\Sigma$:
\begin{equation}\label{eq:gaussian_bound}
    H(s_{1:T}) \;\leq\; H\!\bigl(\mathcal{N}(\,\cdot\,,\,\Sigma)\bigr) \;=\; \tfrac{1}{2}\log\det(2\pi e\,\Sigma), \qquad \Sigma = C + \sigma^2 I_{Td}.
\end{equation}
The epistemic part $C \in \mathbb{R}^{Td \times Td}$ has $d \times d$ blocks
\begin{equation}\label{eq:epistemic_block}
    C_{jt} = \frac{1}{M}\sum_{k=1}^{M} \boldsymbol{\delta}_k^{(j)} \bigl(\boldsymbol{\delta}_k^{(t)}\bigr)^{\!\top}, \qquad \boldsymbol{\delta}_k^{(t)} = \mu_k(s_{t-1}, a_{t-1}) - \bar{\mu}(s_{t-1}, a_{t-1}),
\end{equation}
collecting the deviations of each member from the ensemble mean $\bar{\mu} = \tfrac{1}{M}\sum_k \mu_k$.
Maximising this bound directly fails in high dimension. Because $C$ is a sum of $M$ centred outer products, $\operatorname{rank}(C) \leq M-1$ regardless of the latent dimension $d$. A1 has reduced exploration to identifying which of $M$ members generated the trajectory, a problem with at most $\log M$ bits of capacity, so a single observation can saturate the surrogate and leave nothing to reward at later steps. The fix is design choice \textbf{D1 (trace reduction)}: we collapse each $d \times d$ block of $C$ to its trace, replacing the full epistemic covariance with the kernel
\begin{equation}\label{eq:gram_matrix}
    K_{jt} = \operatorname{tr}(C_{jt}) = \frac{1}{M}\sum_{k=1}^{M} \bigl(\boldsymbol{\delta}_k^{(j)}\bigr)^{\!\top} \boldsymbol{\delta}_k^{(t)}.
\end{equation}
This raises the effective capacity from $M-1$ to $\min(T,(M{-}1)d)$ (Theorem~\ref{thm:saturation}, Appendix~\ref{app:capacity_saturation}), at the cost of being blind to off-diagonal correlations whose traces vanish. Substituting $K$ into the bound and dropping policy-independent constants gives the \textbf{CIG surrogate}
\begin{equation}\label{eq:approx_objective}
    \tilde{J}(s_{1:T}) = \log\det\tilde{K}, \qquad \tilde{K} \coloneqq K + \sigma^2 d \cdot I_T,
\end{equation}
in which the ridge $\sigma^2 d$ aggregates the aleatoric contribution across the latent dimensions. After D1, $\tilde{J}$ is no longer a formal upper bound on the information gain but retains the structural properties needed of a trajectory-level objective, monotonicity in the PSD order and concavity in $K$ (Appendix~\ref{app:capacity_saturation}).

\subsection{Per-Step Decomposition: The CIG Reward}
\label{sec:reward_derivation}

\begin{figure}
    \centering
    \includegraphics[width=1.0\linewidth]{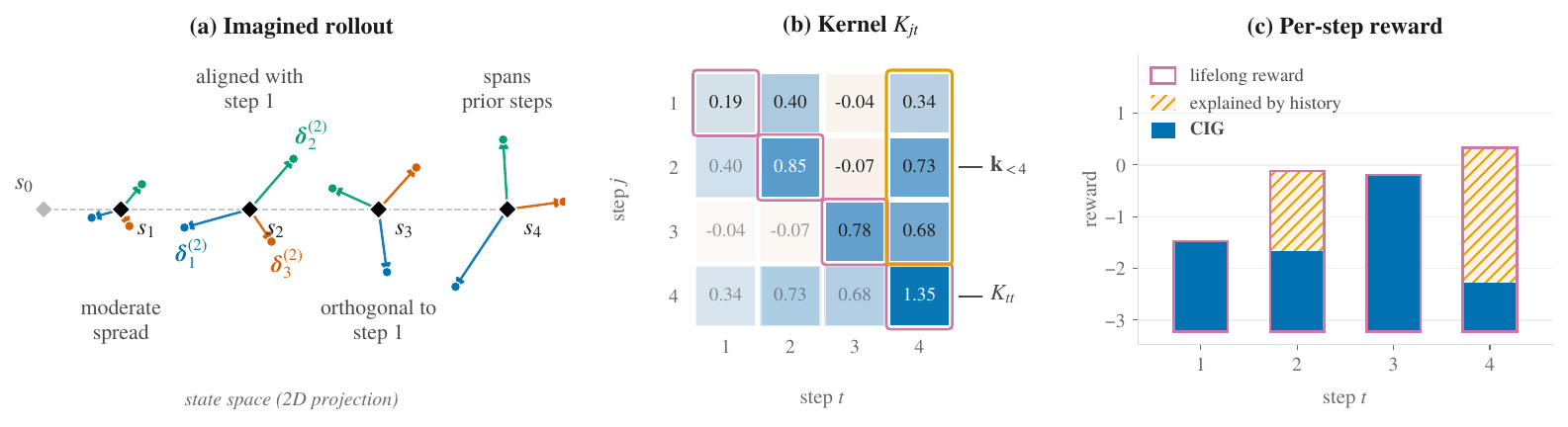}
    \caption{\textbf{Mechanism of the CIG reward on a four-step imagined rollout ($T{=}4$, $M{=}3$).} \textbf{(a)}~Disagreement vectors $\boldsymbol{\delta}_k^{(t)}$ at each step. \textbf{(b)}~The kernel $K_{jt}$ (Eq.~\ref{eq:gram_matrix}) built from their inner products; the column $\mathbf{k}_{<4}$ links step~4 to its prefix. \textbf{(c)}~Eq.~\ref{eq:cig_reward} subtracts the prefix-explained portion from the lifelong term: step~3, nearly orthogonal to the prefix, retains most of its bonus; step~4, largely spanned by the prefix $s_{<4}$, is sharply discounted.}
    \label{fig:method}
\end{figure}

The trajectory objective $\tilde{J}$ admits an exact per-step decomposition via the Cholesky factorisation, making it directly compatible with actor-critic training without trajectory-level value heads or high-variance estimators. Writing $\tilde{K} = LL^\top$ gives $\log\det\tilde{K} = 2\sum_t \log L_{tt}$, with each $L_{tt}^2$ equal to the Schur complement of the leading $(t{-}1) \times (t{-}1)$ submatrix. Setting $r_t = 2\log L_{tt}$ recovers $\tilde{J}(s_{1:T}) = \sum_t r_t$ exactly, with $r_t$ depending only on steps $1,\dots,t$:
\begin{equation}\label{eq:cig_reward}
    r_t = \log\!\Bigl( \underbrace{K_{tt}}_{\text{Lifelong}} + \sigma^2 d \;-\; \underbrace{\mathbf{k}_{<t}^\top\, \tilde{K}_{<t}^{-1}\, \mathbf{k}_{<t}}_{\text{Prefix Redundancy}} \Bigr),
\end{equation}
where $\mathbf{k}_{<t}$ collects the kernel entries linking step $t$ to its prefix, and $\tilde{K}_{<t}$ is the prefix submatrix. The underlying log-determinant surrogate recovers the form of D-optimal design~\citep{doi:10.1137/1.9780898719109} for deep ensembles.

Eq.~\eqref{eq:cig_reward} composes three mechanisms. \textbf{Lifelong disagreement:} the diagonal $K_{tt}$ is the per-step ensemble disagreement, the same signal as Plan2Explore, carrying the buffer conditioning through the trained ensemble. \textbf{Aleatoric compression:} the ridge $\sigma^2 d$ damps small disagreements toward a constant floor $\log(\sigma^2 d)$, concentrating credit on transitions whose disagreement clearly exceeds the predictor's own noise scale; this is a substantive departure from Plan2Explore, whose bonus scales with $K_{tt}$ alone. \textbf{Prefix redundancy:} the subtracted quadratic form removes the portion of step-$t$ disagreement already explained by the prefix, recovering the chain-rule conditioning on $s_{<t}$; when the step-$t$ disagreement is orthogonal to the prefix, $r_t$ keeps the full lifelong bonus, and when it lies in the prefix span, $r_t$ falls to the aleatoric floor (Figure~\ref{fig:method}c; Proposition~\ref{prop:reward_limits}). The two failure modes of \S\ref{sec:bg_proxies} are addressed symmetrically, with the prefix term preventing the overcounting that lifelong rewards suffer on correlated frontiers and the diagonal preventing the uncertainty-blindness of episodic rewards in familiar-but-uncertain regions. Appendix~\ref{app:prefix_redundancy} characterises this selectivity quantitatively: the correction concentrates on frontier steps where $K_{tt}$ is large, leaves low-disagreement steps untouched, and sharpens as training narrows the frontier.

\subsection{Reward Computation}
\label{sec:training}

\paragraph{Aleatoric scale.} The ensemble is trained by MSE and carries no explicit noise parameter. We estimate $\sigma^2$ post-hoc from ensemble-mean residuals on training transitions and substitute the estimate into Eq.~\eqref{eq:cig_reward} at reward-computation time (Appendix~\ref{app:sigma_estimation}). The ensemble training protocol is identical to Plan2Explore~\citep{sekarPlanningExploreSelfSupervised2020}, so that differences against this baseline isolate the effect of the reward rather than changes to the learned representation.

\paragraph{Complexity.} The per-step rewards drive a standard actor-critic on imagined rollouts (\S\ref{app:impl}). After the ensemble forward passes, assembling the $T \times T$ kernel $K$ from the deviation vectors costs $\mathcal{O}(T^2 M d)$ operations, and the Cholesky factorisation of $\tilde{K}$ costs $\mathcal{O}(T^3)$; the per-step rewards are then read from the diagonal of $L$. Both costs are negligible relative to the $\mathcal{O}(TMd)$ ensemble forward passes for the rollout length we consider ($T = 15$). Pseudocode is given in Algorithm~\ref{alg:cig} (see Appendix~\ref{app:algorithm}).
\section{Experiments}
\label{sec:experiments}

We evaluate CIG on twelve tasks across discrete (Minigrid~\citep{chevalier-boisvert2023minigrid}) and continuous-control (OGBench~\citep{parkOGBenchBenchmarkingOffline2025} domains, including both clean environments and Noisy-TV variants with action-dependent stochastic distractors, to test three claims: that CIG improves exploration coverage over both lifelong and episodic baselines under matched compute and tuning budget; that the ensemble-grounded reward is robust to distractors that degrade count-based and entropy-based signals; and that each component of the CIG reward, the trace reduction~D1, the prefix correction of Eq.~\eqref{eq:cig_reward}, and the aleatoric ridge~$\hat{\sigma}^2$, contributes to the observed gains. Sections~\ref{sec:exp_minigrid}--\ref{sec:exp_cc} address the first two claims per domain, \S\ref{sec:exp_aggregate} reports aggregate statistics, and \S\ref{sec:ablations} isolates individual components.

\subsection{Experimental Setup}
\label{sec:exp_setup}
 
All methods are implemented on top of DreamerV2~\citep{hafnerMasteringAtariDiscrete2021} so that the intrinsic reward is the only axis of variation. Ensemble and MLP architectures, optimizer, and optimizer hyperparameters are identical across methods; baseline-specific parameters (e.g., ICM forward-model coefficient, E3B ridge scale) use the defaults reported in the original publications without further tuning. CIG introduces no additional hyperparameters beyond those of the shared ensemble: the aleatoric scale~$\hat{\sigma}^2$ is estimated post-hoc from ensemble-mean residuals on training transitions (\S\ref{sec:training}), not tuned. The tuning budget is therefore
identical across CIG and P2E; any performance difference is attributable to the reward computation alone. We compare against six baselines spanning the lifelong/episodic axis of \S\ref{sec:bg_proxies}: P2E~\citep{sekarPlanningExploreSelfSupervised2020}, RND~\citep{burdaExplorationRandomNetwork2019}, ICM~\citep{pathakCuriosityDrivenExplorationSelfSupervised2017}, APT~\citep{liuBehaviorVoidUnsupervised2021}, E3B~\citep{henaffExplorationEllipticalEpisodic2022} and E3B~$\times$~P2E~\citep{henaffStudyGlobalEpisodic2023} . We evaluate on twelve tasks across two domains: six MiniGrid~\citep{chevalier-boisvert2023minigrid} tasks (three clean, three Noisy-TV) testing procedurally generated combinatorial exploration, and six continuous-control tasks (four clean, two Noisy-TV) testing spatial coverage, multi-object manipulation, and combinatorial configuration coverage. The Noisy-TV (NT) variants augment observations with an action-dependent stochastic distractor. Each (method, task) pair is run with five seeds; we report the inter-quartile mean (IQM) with 95\,\% stratified bootstrap confidence intervals ~\citep{agarwalDeepReinforcementLearning2021}. Full hyperparameter tables, environment details, baseline implementations, and compute budget are given in Appendix~\ref{app:implementation}.

\subsection{MiniGrid}
\label{sec:exp_minigrid}

\begin{figure}[t]
\centering
\includegraphics[width=\linewidth]{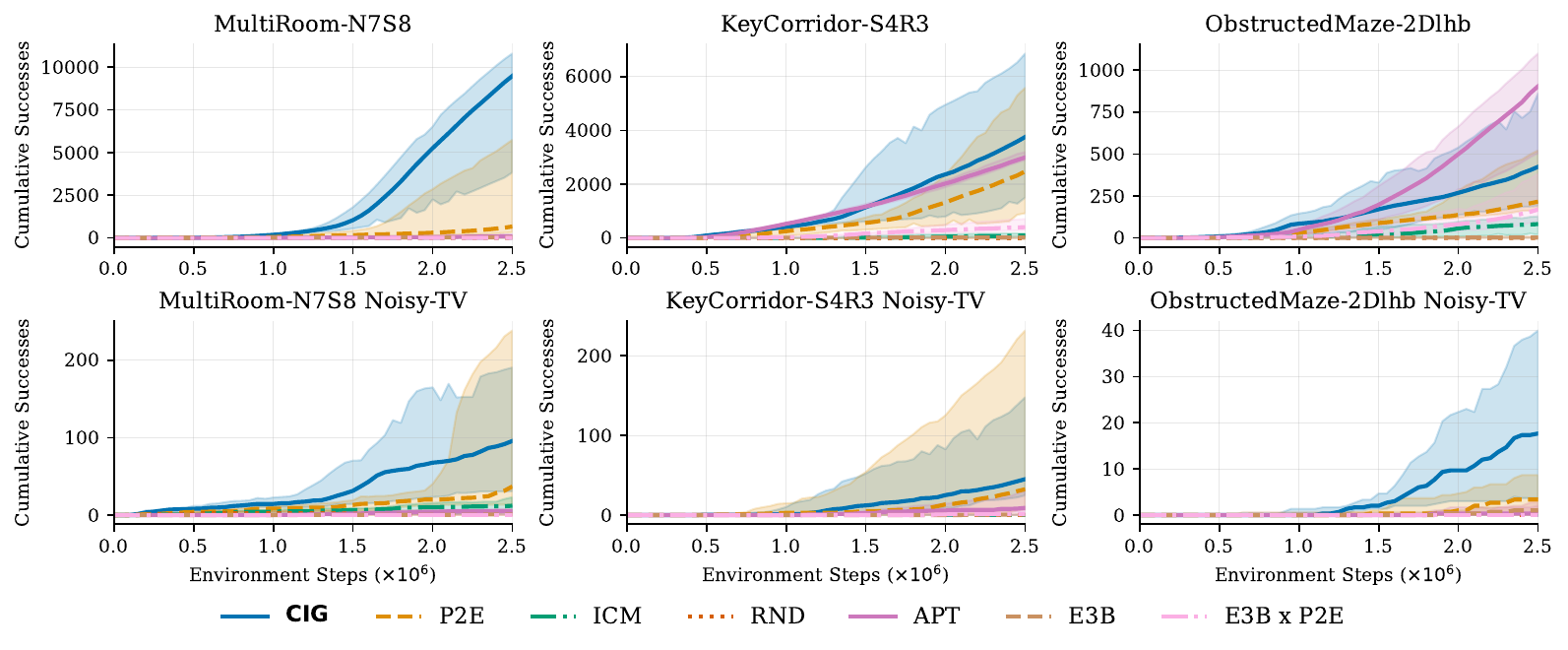}
\caption{%
Cumulative successes (total completed episodes) on three MiniGrid tasks (columns) in clean (top) and Noisy-TV (bottom) variants. Lines show the IQM across five seeds; shaded regions are 95\,\% stratified bootstrap confidence intervals. Under Noisy-TV distractors, non-ensemble methods drop to near-zero success rates while CIG and P2E accumulate successes throughout training.%
}
\label{fig:minigrid}
\end{figure}

MiniGrid~\citep{chevalier-boisvert2023minigrid} provides procedurally generated gridworlds in which layouts and object positions are resampled every episode, and solving a task requires interacting with objects in a specific order. We evaluate on three distinct tasks (MultiRoom~N7S8, KeyCorridor~S4R3, and ObstructedMaze~2DlhB), each in a clean and a Noisy-TV variant, and report cumulative successes over training. Environment details are in Appendix~\ref{sec:minigrid-envs}.

\paragraph{Clean variants.}
Figure~\ref{fig:minigrid} (top row) shows cumulative successes on the three clean tasks. CIG and P2E share the same ensemble, so the gap between them is attributable to the reward computation alone. On MultiRoom~N7S8, CIG accumulates an order of magnitude more successes than P2E by the end of training, and its success rate continues to accelerate where P2E plateaus. On KeyCorridor~S4R3, CIG and APT accumulate successes at a comparable rate throughout training, with P2E trailing both but continuing to improve; the three methods' confidence intervals overlap toward the end of the budget. On ObstructedMaze~2DlhB, APT accumulates roughly twice the successes of CIG: its state-entropy proxy produces high intra-episode state diversity, which is effective when uniform coverage aligns with the task structure. A per-method entropy analysis (Appendix~\ref{app:entropy_success}) shows that high visit-distribution entropy alone does not predict downstream performance: APT and E3B achieve the highest entropy across tasks, yet only APT on ObstructedMaze translates this into an clear advantage over other baselines, while CIG outperforms with lower entropy on the other two tasks, pointing to targeted uncertainty reduction as the driving factor. E3B~$\times$~P2E trails standalone P2E on MultiRoom and KeyCorridor, indicating that na\"ively multiplying an episodic indicator with a lifelong bonus can suppress the lifelong signal rather than refine it.

\paragraph{Noisy-TV variants.}
Ensemble disagreement vanishes for irreducible stochastic transitions, so CIG and P2E are immune to the distractor by construction. RND and APT score every novel observation regardless of source; ICM discards action-independent variation but not the action-dependent noise this variant introduces. Figure~\ref{fig:minigrid} (bottom row) confirms these predictions: RND, ICM, APT, E3B, and E3B~$\times$~P2E drop to near-zero successes on all three tasks. APT's clean-task advantage on ObstructedMaze does not survive the distractor. Only CIG and P2E retain non-trivial success counts, with CIG outpacing P2E on every task. Because both methods already filter irreducible noise through the ensemble, this persistent gap reflects the additional structure CIG's reward formulation captures beyond per-step disagreement.

\subsection{Continuous Control}
\label{sec:exp_cc}

\begin{figure}[t]
\centering
\includegraphics[width=\linewidth]{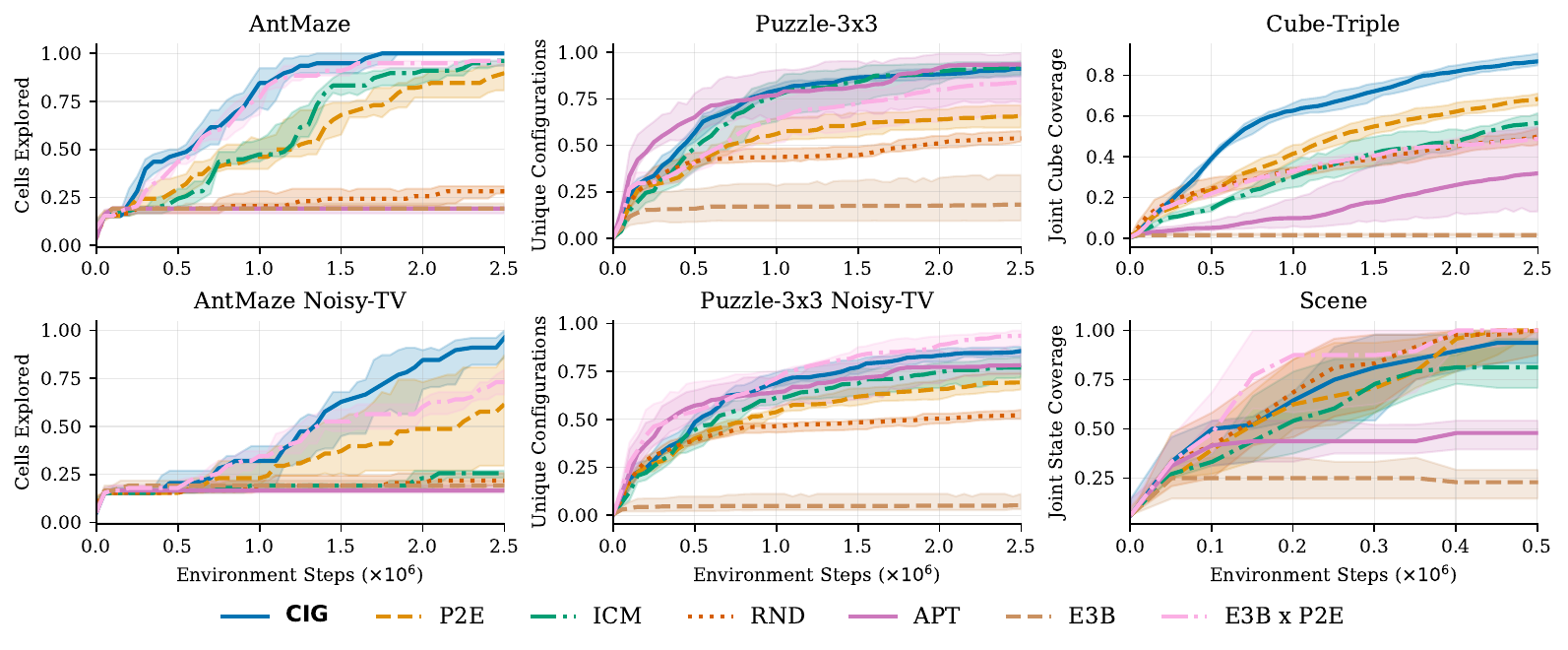}
\caption{%
Coverage curves on six continuous-control tasks: four clean (top row and bottom right) and two Noisy-TV (bottom left). Lines show the IQM across five seeds; shaded regions are 95\,\% stratified bootstrap confidence intervals. CIG is the only method that remains competitive across all six panels; other methods match or exceed CIG on individual tasks but degrade on at least one.%
}
\label{fig:continuous}
\end{figure}

We evaluate on four continuous-control tasks that pose complementary exploration challenges: spatial coverage under quadrupedal locomotion (AntMaze), where exploration is bottlenecked by learning an 8-DoF gait before coverage can accumulate; continuous multi-object configuration coverage (Cube-Triple); coverage of a combinatorial discrete configuration space where each button press toggles itself and its neighbours (Puzzle~$3{\times}3$); and coverage over a heterogeneous set of manipulation objects (Scene). AntMaze and Puzzle~$3{\times}3$ additionally admit Noisy-TV variants. Task-specific coverage metrics are defined in Appendix~\ref{sec:ogbench-envs}.

\paragraph{Clean variants.}
Figure~\ref{fig:continuous} (top row and bottom right) shows coverage on the four clean tasks. On AntMaze, CIG saturates coverage first; E3B~$\times$~P2E follows closely, and P2E and ICM trail with overlapping confidence intervals. APT, E3B, and RND plateau early at low coverage, never progressing beyond the initial region of the maze. On Cube-Triple, CIG opens the widest margin over all baselines: P2E is the nearest competitor, and E3B~$\times$~P2E falls below standalone P2E. Both AntMaze and Cube-Triple exhibit high within-rollout redundancy (consecutive steps often probe the same locomotion gap or move the robot instead of the cubes), and CIG's gains over P2E are largest in these settings.
On Puzzle~$3{\times}3$, CIG, APT, and ICM reach comparable final coverage with overlapping confidence intervals; E3B~$\times$~P2E is close behind, and P2E trails all four. Each button press tends to yield a distinct configuration and the buttons must be pressed in the right order to reach unexplored configurations. APT reaches similar coverage through a different mechanism: it accumulates roughly an order of magnitude more button flips than CIG (Appendix~\ref{app:puzzle_flips}), relying on interaction volume rather than selectivity. On Scene, the small configuration space is covered quickly by most methods. P2E and RND reach full coverage before CIG, bounding CIG's advantage to tasks where within-rollout correlation is non-trivial; when the exploration problem is already easy, CIG's more structured reward provides no additional benefit.

\paragraph{Noisy-TV variants.}
Figure~\ref{fig:continuous} (bottom left two panels) extends the distractor analysis of \S\ref{sec:exp_minigrid} to continuous control. On AntMaze Noisy-TV, CIG retains nearly all of its clean-task coverage; P2E and E3B~$\times$~P2E both degrade but remain above the plateau of RND, ICM, APT, and E3B, which stall at or near their initial coverage. CIG's lead over P2E on this panel mirrors the MiniGrid pattern: both methods filter the distractor through the ensemble, but CIG's reward formulation extracts a stronger signal from the remaining epistemic uncertainty. E3B~$\times$~P2E degrades more from its clean-task level than CIG or P2E do, but it does not collapse; its performance falls between P2E and CIG.
On Puzzle~$3{\times}3$ Noisy-TV, the distractor has a smaller effect than on other tasks: CIG, P2E, and RND retain most of their clean-task coverage, while APT and ICM degrade moderately. E3B~$\times$~P2E is the exception in both directions: it is the only method whose coverage increases relative to its clean variant, reaching the highest final coverage of any method on this panel. We do not have a clear explanation for this improvement; it is the only panel on which E3B~$\times$~P2E surpasses CIG. CIG's coverage drops modestly between variants, a smaller relative degradation than APT or ICM, confirming that CIG's mechanisms compose robustly without task-specific calibration.

\subsection{Aggregate analysis}
\label{sec:exp_aggregate}

\begin{figure}[t]
\centering
\includegraphics[width=\linewidth]{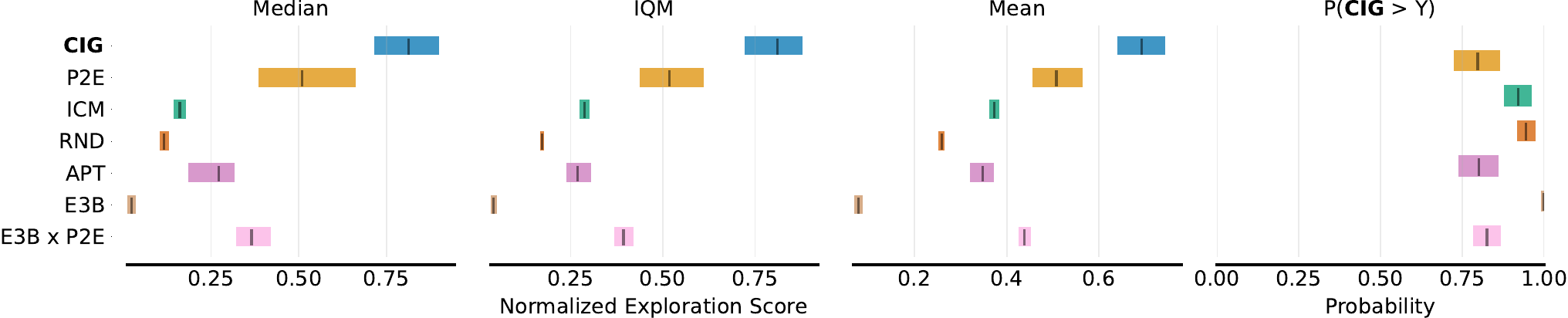}
\caption{Aggregate normalised exploration scores across all twelve tasks (six MiniGrid, six continuous control). Per-task scores are normalised to $[0,1]$ by the best score on that task. Left: median, IQM, and mean with 95\,\% stratified bootstrap confidence intervals over five seeds per (method,\,task) pair~\citep{agarwalDeepReinforcementLearning2021}. Right: probability that CIG outperforms each baseline on a randomly sampled task. CIG leads on all three aggregate statistics; its IQM confidence interval does not overlap with that of any other baseline.}
\label{fig:aggregate}
\end{figure}

The per-domain results of \S\ref{sec:exp_minigrid}--\ref{sec:exp_cc} show that CIG outpaces baselines on most individual tasks; we now ask whether this advantage is consistent across the full suite. Figure~\ref{fig:aggregate} (left) reports normalised scores across all twelve tasks. CIG leads on all three statistics with non-overlapping IQM confidence intervals against every other method. CIG's median and IQM are closely aligned, confirming that the advantage is not driven by a few high-scoring tasks; its lower mean reflects the tasks where CIG does not lead (ObstructedMaze clean, Puzzle~$3{\times}3$, Scene; \S\ref{sec:exp_minigrid}--\ref{sec:exp_cc}). P2E's three statistics cluster tightly, reflecting moderate but uniform performance. APT's mean exceeds its median, consistent with strong wins on a subset of tasks that do not transfer across the suite. E3B~${\times}$~P2E falls below standalone P2E in aggregate, confirming that heuristic combination does not compensate for the hybrid's failure modes on Cube-Triple and AntMaze Noisy-TV. 
E3B, the only purely episodic baseline, ranks last in aggregate across all three statistics, consistent with the short-rollout limitation identified in \S\ref{sec:bg_proxies}: with $T \approx 15$ imagined steps the prefix carries too little information for the episodic correction to produce a useful signal. Figure~\ref{fig:aggregate} (right) reports the probability of improvement~\citep{agarwalDeepReinforcementLearning2021}, the probability that CIG outscores a given baseline on a uniformly sampled task (0.5 = parity). All values exceed 0.79 with CIs bounded away from 0.5, confirming that CIG's advantage holds at the level of individual tasks.

\subsection{Ablations}
\label{sec:ablations}
\begin{wrapfigure}{r}{0.42\textwidth}
\centering
\vspace{-20pt}
\includegraphics[width=\linewidth]{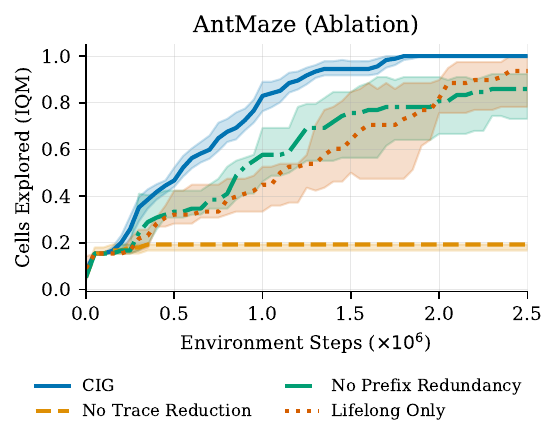}
\caption{Ablation on AntMaze. Each variant removes one component of the CIG reward (Eq.~\eqref{eq:cig_reward}). Lines show the IQM; shaded regions are 95\,\% stratified bootstrap confidence intervals.}

\label{fig:ablation}
\vspace{-10pt}
\end{wrapfigure}
We isolate the three components of Eq.~\eqref{eq:cig_reward} on AntMaze in Figure~\ref{fig:ablation}. Removing the trace reduction (\emph{No Trace Reduction}) is the most severe intervention: coverage flatlines early in training, with CIs that separate from all other variants and never recover. This confirms the rank-$(M{-}1)$ capacity saturation derived in \S\ref{sec:surrogate}: the surrogate's limited rank renders the reward insensitive to meaningful variation across steps, and the trace reduction restores sufficient fidelity for the policy to distinguish explored from unexplored regions. The prefix correction is the main source of sample efficiency. Both variants that remove it learn steadily but converge more slowly, with CIG's lower CI above both upper CIs through mid-training. On AntMaze, where consecutive steps revisit the same locomotion frontier, the prefix correction concentrates credit on steps that probe new directions in disagreement space. The ridge is not independently separable: the two variants (\emph{No Prefix Redundancy} and \emph{Lifelong Only}) that differ only by the ridge have overlapping CIs at every checkpoint, with point estimates that reverse the expected ordering by the end of training. This is consistent with the design of Eq.~\eqref{eq:cig_reward}, where the ridge sets the floor toward which the prefix correction compresses redundant steps; its effect is mediated by the prefix term and may be small when that term is absent.
\section{Conclusion}
\label{sec:conclusion}

Lifelong and episodic exploration rewards each discard half of the available context: the former ignores what has already been observed within a rollout, while the latter forgets what was learned across rollouts. Prior attempts to combine both signals have relied on heuristic products or Gaussian-process dynamics that do not scale. CIG shows that the two conditioning sets need not be stitched together post hoc: they emerge jointly from a single log-determinant surrogate of trajectory-level information gain, whose Cholesky factorisation produces causal, per-step rewards with no additional design choices. The practical consequence is robustness across problem structure. On tasks where count- and entropy-based baselines collapse (all Noisy-TV variants), the shared ensemble posterior filters irreducible noise; on tasks where most methods already perform well (Puzzle~$3{\times}3$, Scene), CIG reaches comparable coverage. Across the full twelve-task suite, CIG achieves the highest aggregate IQM with non-overlapping confidence intervals against every baseline (\S\ref{sec:exp_aggregate}). No other method avoids degradation across all conditions we tested.

\paragraph{Limitations and future work.}
CIG does not dominate on every task: APT outperforms it on ObstructedMaze (clean), and E3B$\times$P2E surpasses it on Puzzle~$3{\times}3$ Noisy-TV for reasons we do not yet fully understand. Characterising the problem structure under which CIG's surrogate is loose enough for simpler proxies to win is an open question. The evaluation is restricted to model-based RL with short imagined rollouts (DreamerV2, $T \approx 15$). In model-free settings with episodes spanning hundreds of steps, the $\mathcal{O}(T^3)$ Cholesky cost and $T \times T$ kernel storage may become prohibitive, and the posterior signal from a small ensemble ($M = 5$) would compound errors over longer horizons; sliding-window or low-rank kernel approximations are a natural next step. More broadly, the trace reduction discards directional structure in the epistemic covariance; retaining partial directional information without reintroducing the rank-saturation bottleneck of \S\ref{sec:surrogate} remains open.

\bibliographystyle{unsrtnat}
\bibliography{CIG}


\appendix

\clearpage
\appendix

\section{Implementation}
\label{app:implementation}

\subsection{Pseudocode}
\label{app:algorithm}

\begin{algorithm}[t]
\caption{Conditional Information Gain reward (Algorithm referenced in §\ref{sec:training}).}
\label{alg:cig}
\begin{algorithmic}[1]
\Require Imagined latent rollout $(s_0, a_0, \dots, s_{T-1}, a_{T-1})$;
         ensemble of one-step predictors $\{\mu_k\}_{k=1}^{M}$;
         aleatoric estimate $\hat{\sigma}^2$; latent dimension $d$.
\Ensure Per-step intrinsic rewards $r_{1:T}$.
\For{$t = 1, \dots, T$}
    \For{$k = 1, \dots, M$}
        \State $m_k^{(t)} \gets \mu_k(s_{t-1}, a_{t-1})$
            \Comment{ensemble forward pass (A2)}
    \EndFor
    \State $\bar{m}^{(t)} \gets \tfrac{1}{M}\sum_{k=1}^{M} m_k^{(t)}$
    \State $\boldsymbol{\delta}_k^{(t)} \gets m_k^{(t)} - \bar{m}^{(t)} \quad \text{for } k=1,\dots,M$
        \Comment{centred deviations, Eq.~\eqref{eq:epistemic_block}}
\EndFor
\State $K_{jt} \gets \tfrac{1}{M}\sum_{k=1}^{M}
        \bigl(\boldsymbol{\delta}_k^{(j)}\bigr)^{\!\top} \boldsymbol{\delta}_k^{(t)}
        \quad \text{for } 1 \le j, t \le T$
        \Comment{trace-reduced Gram (D1), Eq.~\eqref{eq:gram_matrix}}
\State $\tilde{K} \gets K + \hat{\sigma}^2 d \cdot I_T$
        \Comment{aleatoric ridge, Eq.~\eqref{eq:approx_objective}}
\State $L \gets \mathrm{Cholesky}(\tilde{K})$
        \Comment{$\tilde{K} = L L^{\!\top}$}
\For{$t = 1, \dots, T$}
    \State $r_t \gets 2\log L_{tt}$
        \Comment{Schur complement = Eq.~\eqref{eq:cig_reward}}
\EndFor
\State \Return $r_{1:T}$
\end{algorithmic}
\end{algorithm}

\begin{algorithm}[t]
\caption{Aleatoric scale update (Appendix~\ref{app:sigma_estimation}).}
\label{alg:sigma}
\begin{algorithmic}[1]
\Require Mini-batch MSE residuals $\{\epsilon_i\}_{i=1}^{B}$ between ensemble means and targets; decay $\beta \in [0,1)$; running estimate $\hat{\sigma}^2$ (initialised to $0$).
\State $\bar{\epsilon} \gets \tfrac{1}{B}\sum_{i=1}^{B} \epsilon_i$
\If{$\hat{\sigma}^2 = 0$} \Comment{first update: take the batch value directly}
    \State $\hat{\sigma}^2 \gets \bar{\epsilon}$
\Else
    \State $\hat{\sigma}^2 \gets \beta\,\hat{\sigma}^2 + (1-\beta)\,\bar{\epsilon}$
\EndIf
\end{algorithmic}
\end{algorithm}

Algorithm~\ref{alg:cig} consolidates the reward computation of
§\ref{sec:method} into explicit pseudocode and computes the per-step
rewards $r_{1:T}$ for a single imagined rollout. The three substantive steps mirror
§\ref{sec:method}: lines~5--10 evaluate the ensemble and form the
centred deviations $\boldsymbol{\delta}_k^{(t)}$ of
Eq.~\eqref{eq:epistemic_block} (assumptions A1--A2,
§\ref{sec:marginal_entropy}); line~11 contracts each $d \times d$
block of the epistemic covariance to its trace to obtain the kernel
$K$ of Eq.~\eqref{eq:gram_matrix} (design choice D1,
§\ref{sec:surrogate}); lines~12--16 recover the per-step rewards from
the Cholesky factorisation of the ridged kernel
$\tilde{K} = K + \hat{\sigma}^2 d \cdot I_T$. The Schur-complement
identity discussed in §\ref{sec:reward_derivation} guarantees that
$r_t = 2 \log L_{tt}$ realises Eq.~\eqref{eq:cig_reward} exactly.

Algorithm~\ref{alg:sigma} maintains the aleatoric estimate
$\hat{\sigma}^2$ used as the ridge in Algorithm~\ref{alg:cig}. As
described in §\ref{sec:training}, we estimate it post-hoc from the
MSE residuals of the ensemble mean on training transitions, and
update it with an exponential moving average with each training step;
the running buffer is initialised to zero, in which case the first
batch value is taken directly to avoid biasing the estimator toward
the initialisation.

\subsection{Architecture and Training}
\label{app:impl}

CIG and all baselines share a common DreamerV2
backbone~\citep{hafnerMasteringAtariDiscrete2021} and differ only in the form of the
intrinsic reward. CIG introduces no architectural changes to the backbone. 

\paragraph{Architecture.}
The world model follows the Recurrent State-Space Model
(RSSM)~\citep{hafnerMasteringAtariDiscrete2021} with a discrete $32 \times 32$ stochastic
state space. A convolutional encoder maps pixel observations to the
latent state; a convolutional decoder reconstructs observations from
the concatenation of the deterministic and stochastic components.
Since we operate in a reward-free setting, no reward predictor is
attached. The actor and critic are MLPs conditioned on the
concatenated latent: the actor outputs a truncated Normal
distribution for continuous action spaces and a categorical
distribution for discrete action spaces; the critic outputs a scalar
state value. The ensemble consists of $M = 5$ independently
initialised MLPs operating in the latent space of the world model.
All members receive the same training data and differ only in their
random initialisation; we use no bootstrap masks, diversity
penalties, or architectural variation between members. The ensemble
is used by CIG and Plan2Explore; remaining baselines do not require
it, but instead have different architectures that include MLPs of the same depth and width as the ensemble MLPs.

\paragraph{Training.}
All components are trained with Adam~\citep{2015-kingma}. The world
model is trained on replay-buffer sequences by jointly optimising
the reconstruction, KL-regularisation, and discount losses of
DreamerV2. The ensemble is trained alongside the world model via
per-member mean-squared error on replay-buffer transitions encoded
to latent states; the reward path detaches the ensemble outputs, so
the intrinsic reward does not backpropagate into the ensemble
parameters. The actor and critic are trained on imagined rollouts of
the world model: starting from replay-sampled latents, the actor is
rolled out through the RSSM transition model for a fixed horizon,
and the intrinsic reward is computed at each imagined step. The
critic is regressed onto $\lambda$-returns; the actor is updated
through the world-model dynamics via reparameterisation gradients
(continuous action spaces) or REINFORCE (discrete action spaces).
The behaviour policy is the actor itself, with no additional noise other than that induced by entropy bonus that is part of DreamerV2, thus exploration is mainly driven by the intrinsic reward. 

\paragraph{Baselines.}
All baselines use the same DreamerV2 world model, actor--critic
architecture, and training procedure described above. Baselines that
require ensemble disagreement (Plan2Explore) use the same ensemble.
Each baseline reward is computed on imagined rollouts in the same
latent space, isolating the reward as the sole variable of
comparison. The specific baseline rewards are listed in
\S\ref{sec:baselines}.

\subsection{Hyperparameters}
\label{sec:hparams}

CIG is built on top of a DreamerV2 world model~\citep{hafnerMasteringAtariDiscrete2021}
augmented with a Plan2Explore~\citep{sekarPlanningExploreSelfSupervised2020} ensemble. We follow
the default DreamerV2 backbone configuration of \citet{hafnerMasteringAtariDiscrete2021}
unchanged and refer the reader there for full detail; the most
important values are restated in Table~\ref{tab:hparams} alongside the
hyperparameters introduced by CIG and the Plan2Explore ensemble. 

\begin{table}[h]
\centering
\caption{Hyperparameters for the CIG experiments. Backbone settings
not listed here follow the DreamerV2 defaults of \citet{hafnerMasteringAtariDiscrete2021}.}
\label{tab:hparams}
\small
\begin{tabular}{lll}
\toprule
\textbf{Hyperparameter} & \textbf{Symbol} & \textbf{Value} \\
\midrule
\multicolumn{3}{l}{\textit{CIG reward (Eq.~\ref{eq:approx_objective}, Eq.~\ref{eq:cig_reward})}} \\
\midrule
Ridge scale multiplier              & ---                & $1.0$ \\
$\hat{\sigma}^2$ EMA momentum       & $\beta_{\sigma^2}$ & $0.99$ \\
\midrule
\multicolumn{3}{l}{\textit{Plan2Explore ensemble (carries A1)}} \\
\midrule
Ensemble size                       & $M$ & $5$ \\
Target dimensionality               & $d$ & $200$ (proprioceptive) / $1536$ (pixels) \\
Head architecture                   & --- & shared MLP (below) \\
\midrule
\multicolumn{3}{l}{\textit{Shared MLP architecture and optimization}} \\
\midrule
Hidden layers                       & --- & $4$ \\
Hidden units per layer              & --- & $400$ \\
Activation                          & --- & ELU \\
Initialization                      & --- & Glorot \\
Optimizer                           & --- & Adam, $\epsilon=10^{-5}$, wd $10^{-6}$, grad-clip $100$ \\
\midrule
\multicolumn{3}{l}{\textit{DreamerV2 world model}} \\
\midrule
Stochastic latent (categorical)     & ---       & $32 \times 32$ \\
Deterministic GRU state             & ---       & $600$ \\
KL balancing                        & $\alpha$  & $0.8$ \\
Free nats                           & ---       & $0.0$ \\
CNN channel multiplier              & ---       & $48$ \\
Model learning rate                 & ---       & $2 \times 10^{-4}$ \\
\midrule
\multicolumn{3}{l}{\textit{Actor--critic over imagined rollouts}} \\
\midrule
Imagination horizon                 & $T$       & $15$ \\
Discount                            & $\gamma$  & $0.99$ \\
$\lambda$-return                    & $\lambda$ & $0.95$ \\
Actor learning rate                 & ---       & $8 \times 10^{-5}$ \\
Critic learning rate                & ---       & $8 \times 10^{-5}$ \\
Slow-critic update                  & ---       & hard, every $100$ gradient steps \\
\midrule
\multicolumn{3}{l}{\textit{Replay buffer and training schedule}} \\
\midrule
Replay capacity                     & --- & $10^6$ transitions \\
Sequence length                     & --- & $50$ \\
Batch size (trajectories)           & --- & $50$ \\
Random prefill                      & --- & $5{,}000$ env steps \\
Train every                         & --- & $100$ env steps \\
Gradient updates per train call     & --- & $10$ \\
\midrule
\multicolumn{3}{l}{\textit{Reward normalization}} \\
\midrule
Per-step reward normalization       & --- & z-score \\
Running statistics EMA momentum     & --- & $0.99$ \\
Return-level normalization          & --- & Welford running return std \\
\bottomrule
\end{tabular}
\end{table}

\subsection{Environments}
\label{sec:environments}

\begin{figure*}[t]
  \centering
  \setlength{\tabcolsep}{1pt}
  \resizebox{\textwidth}{!}{%
  \begin{tabular}{cccccc}
    \makebox[0pt]{\tiny KeyCorridorS4R3} & \makebox[0pt]{\tiny MultiRoom-N7-S8} & \makebox[0pt]{\tiny ObstructedMaze-2Dlhb} & \makebox[0pt]{\tiny AntMaze} & \makebox[0pt]{\tiny Puzzle-3$\times$3} & \makebox[0pt]{\tiny Scene Explore} \\
    \includegraphics[width=0.158\textwidth]{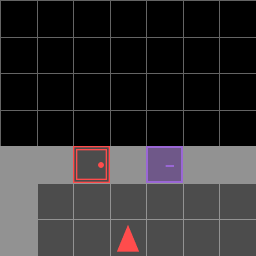} &
    \includegraphics[width=0.158\textwidth]{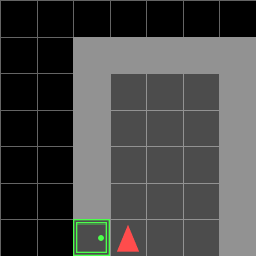} &
    \includegraphics[width=0.158\textwidth]{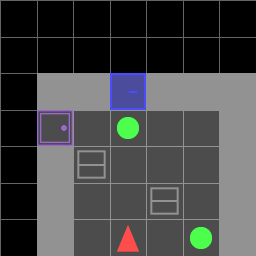} &
    \includegraphics[width=0.158\textwidth]{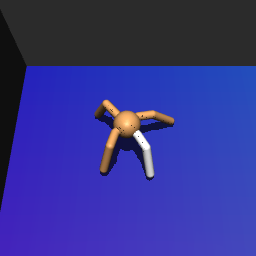} &
    \includegraphics[width=0.158\textwidth]{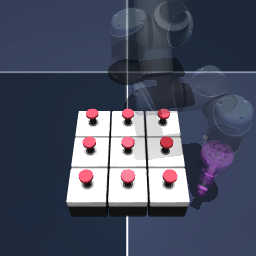} &
    \includegraphics[width=0.158\textwidth]{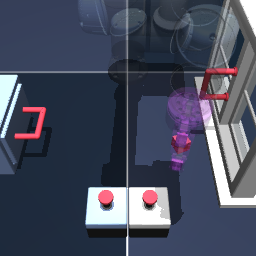} \\[4pt]
    \makebox[0pt]{\tiny KeyCorridorS4R3-NT} & \makebox[0pt]{\tiny MultiRoom-N7-S8-NT} & \makebox[0pt]{\tiny ObstructedMaze-2Dlhb-NT} & \makebox[0pt]{\tiny AntMaze-NT} & \makebox[0pt]{\tiny Puzzle-3$\times$3-NT} & \makebox[0pt]{\tiny Cubes Triple} \\
    \includegraphics[width=0.158\textwidth]{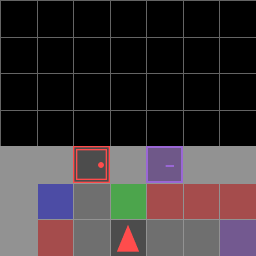} &
    \includegraphics[width=0.158\textwidth]{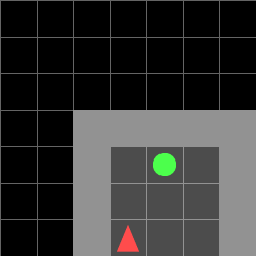} &
    \includegraphics[width=0.158\textwidth]{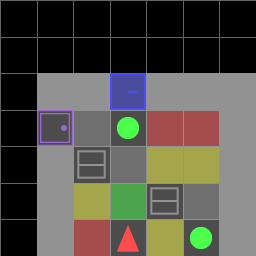} &
    \includegraphics[width=0.158\textwidth]{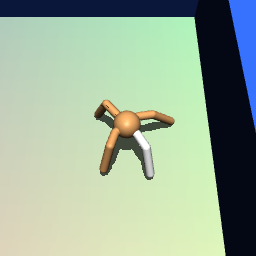} &
    \includegraphics[width=0.158\textwidth]{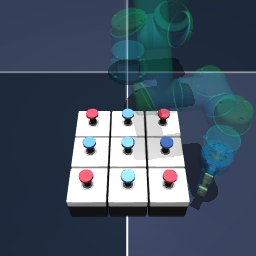} &
    \includegraphics[width=0.158\textwidth]{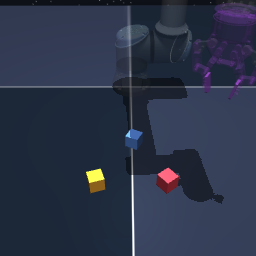} \\
  \end{tabular}%
  }
  \caption{RGB observations from the evaluation environments (the agent sees 64$\times$64 downsampled versions).}
  \label{fig:env-observations}
\end{figure*}

All environments operate in a fully reward-free setting: the agent never
observes extrinsic reward during pretraining.  Instead, we track a
domain-specific coverage proxy for each environment that measures how much
of the reachable state space the agent has visited.  Every environment
provides $64{\times}64$ RGB (see Figure \ref{fig:env-observations}).
Table~\ref{tab:env-summary} gives an overview; the paragraphs that follow
supply the details needed for reproduction.

\begin{table}[h]
\centering
\caption{Summary of evaluation environments.  ``Distractor'' entries
refer to the visual-distractor variants described in
\S\ref{sec:distractor-variants}.}
\label{tab:env-summary}
\small
\begin{tabular}{llcll}
\toprule
\textbf{Environment} & \textbf{Action space} & \textbf{Ep.\ length} &
  \textbf{Capability tested} & \textbf{Metric} \\
\midrule
\multicolumn{5}{l}{\textit{MiniGrid (discrete, 7 actions)}} \\
KeyCorridorS4R3        & discrete & default & multi-step credit assignment    & cum.\ successes \\
MultiRoom-N7-S8        & discrete & default & long-horizon directed expl.     & cum.\ successes \\
ObstructedMaze-2Dlhb   & discrete & default & compositional object interact.  & cum.\ successes \\
+ distractor variants  & \multicolumn{4}{l}{\quad(same metrics; see \S\ref{sec:distractor-variants})} \\
\midrule
\multicolumn{5}{l}{\textit{OGBench (continuous)}} \\
Scene Explore          & 7-DoF arm   & 1000 & heterogeneous affordances     & sub-task \% \\
AntMaze                & 8-DoF ant   & 1000 & spatial expl.\ + locomotion   & cell \% \\
Puzzle-3$\times$3      & 7-DoF arm   & 1000 & combinatorial coverage        & config.\ \% \\
Cubes Triple           & 7-DoF arm   & 1000 & contact-rich rearrangement    & config.\ \% \\
+ distractor variants  & \multicolumn{4}{l}{\quad(same metrics; see \S\ref{sec:distractor-variants})} \\
\bottomrule
\end{tabular}
\end{table}

\subsubsection{MiniGrid}
\label{sec:minigrid-envs}

All MiniGrid environments use the standard seven discrete actions (turn
left, turn right, move forward, pick up, drop, toggle, done) and render a
$64{\times}64$ RGB egocentric partial view.  Environments are
procedurally generated, so the agent encounters a new layout variant on
every reset.  Episodes are truncated after the per-task default step
budget defined by \citet{chevalier-boisvert2023minigrid}.

The exploration metric shared by all MiniGrid tasks is \emph{cumulative
successes}: at every environment step we check whether the agent has
reached a rewarding state (the goal tile, target object, etc.) and, if so,
record a single success for that episode.  Visiting the rewarding state
more than once within the same episode does not increment the counter
further.  We adopt this metric because the agent never observes the
extrinsic signal, so cumulative successes measure how often reward-free
exploration incidentally solves the task, which is the quantity that
ultimately matters for any downstream exploitation phase.

\paragraph{KeyCorridorS4R3.}
KeyCorridorS4R3 requires picking up a target object
behind a locked door.  The key is placed behind a separate, closed but
unlocked door, so the agent must enter a side room, collect the key,
return to the corridor, and unlock the target door.  The layout consists
of three rows of $4{\times}4$ rooms connected by a central corridor.
The task tests multi-step credit assignment under sparse structure: the
agent must commit to a long prerequisite detour before any successful
trajectory becomes possible, so it reveals whether exploration can
discover and follow a chained prerequisite structure rather than greedily
pursuing local novelty.

\paragraph{MultiRoom-N7-S8.}
MultiRoom-N7-S8 arranges seven rooms of randomised size
(up to $8{\times}8$) in sequence, each connected by a closed door the
agent must toggle open.  The agent starts in the first room and must reach
a green goal tile in the last.  The task tests long-horizon directed
exploration: the agent must keep pushing into unseen rooms rather than
collapsing onto the well-explored neighbourhood of its starting position.

\paragraph{ObstructedMaze-2Dlhb.}
ObstructedMaze-2Dlhb is the second tier of the Obstructed
Maze suite (two rooms across) and the hardest of our three MiniGrid base
tasks.  The suffix indicates that doors are simultaneously locked (require
a coloured key), hidden (the key is concealed inside a box), and blocked
(a ball obstructs the door and must be moved aside).  Reaching the target
therefore requires chaining several object interactions over a long
horizon.  The task tests compositional object interaction: each subgoal
decomposes into distinct manipulation primitives, so the environment
reveals whether exploration scales to settings that require correctly
composing multiple interactions.

\subsubsection{OGBench}
\label{sec:ogbench-envs}

All OGBench environments provide $64{\times}64$ RGB observations and
continuous actions.  Episodes run for a fixed 1000 steps and are truncated
at the end of the budget; there is no early termination.  Initial
configurations are pinned with a fixed reset seed across episodes.  Each
environment defines a coverage metric following the same template: a
continuous state quantity is discretised into cells or configuration
identifiers, each identifier is added to a visited set upon first
encounter, and the reported metric is the fraction of reachable
identifiers visited at least once (expressed as a percentage).  The
visited set is monotone non-decreasing within each evaluation window.
Per-environment details of the discretisation and state quantity are given
below.

\paragraph{Scene Explore.}
The Scene environment is a tabletop manipulation task in
which a 7-DoF arm interacts with two buttons, a drawer, a window, and a
single block.  Both buttons start locked, the drawer and window start
fully closed, and the block is placed at a fixed location.  The task tests
breadth of interaction across heterogeneous affordances: buttons toggle,
the drawer and window slide, and the block can be lifted and translated,
so it reveals whether the policy discovers all available affordances
rather than fixating on the easiest one.
We enumerate a fixed set of sub-tasks defined by the scene's manipulable
elements (each button unlocked, drawer opened, window opened, block
displaced) and mark a sub-task as solved once its ground-truth state
predicate fires for the first time.  The metric is the fraction of
sub-tasks solved.

\paragraph{AntMaze.}
The medium variant of the OGBench antmaze: an 8-DoF
quadruped navigates a fixed maze layout from a top-down view.  The agent
is always reset at the same cell.  The task tests long-horizon spatial
exploration coupled with high-dimensional continuous control from pixels:
the agent must acquire competent locomotion of an 8-DoF body and use that
competence to reach spatially novel regions of the maze.
The maze is partitioned into a grid of cells defined by the underlying
layout.  At every step the ant's torso $(x,y)$ position is snapped to the
enclosing cell and added to the visited set.  The metric is the fraction
of reachable cells visited.

\paragraph{Puzzle-3$\times$3.}
A $3{\times}3$ lights-out style button puzzle.  A 7-DoF
arm presses buttons arranged in a grid; pressing one toggles it and its
orthogonal neighbours, yielding $2^{9}=512$ reachable configurations.
All nine buttons start in the off state.  The task tests combinatorial
state-space coverage: the 512 configurations are accessible only via long
action chains under adjacent-toggle dynamics, so it reveals whether the
agent systematically traverses a discrete combinatorial space rather than
visiting only configurations close to the initial state.
The ground-truth on/off state of all nine buttons is read at every step
and treated as a 9-bit configuration identifier.  The metric is the
fraction of the 512 reachable configurations visited.

\paragraph{Cubes Triple.}
The three-cube tabletop task: a 7-DoF arm interacts with
three cubes initialised in a fixed row.  The task tests configurational
exploration in a continuous, contact-rich domain: the agent must grasp,
lift, and rearrange multiple objects, so it reveals whether the method
discovers diverse multi-object configurations rather than collapsing onto
single-object behaviour.
The workspace is partitioned into a 3D grid (using the cube spawn ranges
in $x$, $y$, $z$ at the spacing used to generate goal positions).  Each
cube's position is snapped to its enclosing cell, and the three cell
indices are concatenated into a configuration identifier.  The metric is
the fraction of unique discretised configurations visited.

\subsubsection{Visual-Distractor Variants}
\label{sec:distractor-variants}

For each domain we construct distractor variants that instantiate the
\emph{noisy-TV} problem with action-dependent noise: a region of the
observation is overwritten with content that is sampled from a
stochastic process the agent cannot predict, but whose updates the
agent itself triggers through a designated action.  The agent therefore
has full control over \emph{when} the distractor refreshes, and no
control over \emph{what} appears, which is exactly the setting that
traps prediction error driven exploration, since each refresh produces
genuinely novel pixels that no learned model can anticipate.  These
variants test whether the method can separate genuine state-space
novelty from this self-induced but irreducible pixel entropy.  The
underlying task, action space, episode budget, and exploration metric
of each base environment are left unchanged; only the visual
observation is modified.  Table~\ref{tab:distractor-variants}
summarises the variants; details follow.

\begin{table}[h]
\centering
\caption{Noisy-TV distractor variants with action-dependent noise.
Each row modifies only the observation of its base environment; all
other aspects (task, actions, episode length, metric) are identical.
The agent triggers each refresh through the listed action, but the
sampled content is unpredictable, so an agent that maximises naive
prediction error can in principle drive unbounded reward through its
own behaviour without ever exploring the task.}
\label{tab:distractor-variants}
\small
\begin{tabular}{lll}
\toprule
\textbf{Base environment} & \textbf{Distractor target} &
  \textbf{Trigger} \\
\midrule
KeyCorridorS4R3      & Floor texture       & \texttt{drop} action \\
MultiRoom-N7-S8      & Colour-changing orb & \texttt{drop} action \\
ObstructedMaze-2Dlhb & Floor texture       & \texttt{drop} action \\
AntMaze              & Wall colours        & Action-dim threshold \\
Puzzle-3$\times$3    & Robot appearance    & Centre-button press \\
\bottomrule
\end{tabular}
\end{table}

\paragraph{MiniGrid distractor variants.}
In KeyCorridorS4R3 and ObstructedMaze-2Dlhb, the floor within the
agent's field of view is overwritten with random colours; the colours
are independently re-sampled whenever the agent issues the
\texttt{drop} action.  In MultiRoom-N7-S8 the floor cannot be modified
because the green goal tile must remain visually identifiable;
instead, we add an orb to the observation whose colour is re-sampled
on the same \texttt{drop} action.  In all three cases the trigger is
task-irrelevant: \texttt{drop} has no effect on the underlying
MiniGrid task in our configurations, so the agent can summon a fresh,
unpredictable visual sample at any time without progressing toward the
goal.  These variants test whether the method still discovers the
prerequisite structures of each base task (key fetching, room
traversal, compositional object interaction) when an action-gated
noisy TV is freely available as an alternative.

\paragraph{AntMaze with action-gated distractor.}
The maze wall colours are re-sampled whenever a designated action
dimension exceeds a threshold; the floor texture, ant body colours,
and lighting are left untouched.  The distractor is therefore
controllable: the agent can summon arbitrarily many fresh wall
colourings without changing its $(x,y)$ position in the maze, but the
sampled colours themselves are unpredictable.  This variant tests
whether the method separates spatial novelty (reaching a new maze
region) from self-induced surface-level pixel change.

\paragraph{Puzzle-3$\times$3 with action-gated robot distractor.}
The robot's rendered appearance (the UR5e arm and Robotiq gripper
materials) is re-sampled whenever the agent presses the centre button;
the puzzle board and surrounding scene are left clean.  The trigger is
itself a puzzle action, but the colour change carries no task-relevant
information: the on/off button configuration that defines the metric
is unaffected, so an agent that maximises naive prediction error can
exploit the distractor by repeatedly toggling the centre button
without exploring the rest of the combinatorial configuration space.
The variant tests whether the method resists such self-induced novelty
and continues to explore the full button-configuration space.

\subsection{Aleatoric Variance Estimation}
\label{app:sigma_estimation}

\paragraph{Ensemble training dynamics}
Assumption A2 posits $s_{t+1} \mid k, s_t, a_t \sim \mathcal{N}(\mu_k(s_t, a_t),\, \sigma^2 I_d)$ with a shared, isotropic, non-learned covariance. We adopt this generative model directly from Plan2Explore~\citep{sekarPlanningExploreSelfSupervised2020} for two reasons. First, it is the probabilistic model implicit in the P2E training protocol, and holding it fixed preserves the P2E ensemble training dynamics exactly; \gls{cig} then differs from P2E only at reward-computation time and in the resulting distribution of training data, so any downstream comparison attributes performance differences to the reward itself. Second, Plan2Explore establishes empirically that this noise model is sufficient for effective ensemble-based exploration in the Dreamer latent space, so we inherit its validation rather than re-litigate it.

The scalar $\sigma^2$ is not optimized during ensemble training and therefore must be supplied separately at reward-computation time, where it enters the ridge $\sigma^2 d \cdot I_T$ in Eq.~\eqref{eq:approx_objective}. We estimate it post-hoc from ensemble-mean residuals.

\paragraph{Estimating $\sigma^2$}
The scalar $\sigma^2$ does not appear in the ensemble training loss and must be estimated post-hoc for use in the ridge term of Eq.~\eqref{eq:approx_objective}. We use residuals against the \emph{ensemble mean} $\bar{\mu}(s_n, a_n) = \frac{1}{M}\sum_k \mu_k(s_n, a_n)$:
\begin{equation}
    \hat{\sigma}^2 = \frac{1}{Nd} \sum_{n=1}^{N} \bigl\|s_{n+1} - \bar{\mu}(s_n, a_n)\bigr\|^2.
    \label{eq:sigma_estimate_app}
\end{equation}
Residuals against individual members would conflate aleatoric noise with epistemic disagreement; the ensemble mean cancels the symmetric component of that disagreement and leaves a residual whose squared norm is dominated by the aleatoric contribution. For finite $M$, $\bar{\mu}$ retains some epistemic error, so $\hat{\sigma}^2$ upper-bounds the true aleatoric floor. This direction of bias is the conservative one: overestimating the aleatoric floor attenuates the conditional correction rather than amplifying it. In practice, $\hat{\sigma}^2$ is maintained as a running average over ensemble-mean residuals on training transitions.

\subsection{Compute \& Wall-clock Time}
\label{app:compute}
All experiments were conducted on machines equipped with a single NVIDIA RTX 3090 GPU (24\,GB VRAM), 32\,GB RAM, and 8 CPU cores. Each individual run completes in under 24 hours, with a typical wall-clock time of approximately 20 hours depending on the level of metric logging. The computational overhead of the proposed intrinsic reward mechanisms is negligible relative to the DreamerV2 base agent, which accounts for the majority of resource consumption. The difference between CIG and P2E, the only other ensemble-based baseline in our experiments, is less than 1\% in FLOPS, since the ensemble forward passes, which are shared by both, take the majority of compute. However, in practice the difference is around 10\% in wall-clock time depending on the GPU architecture and optimizations made by the PyTorch compiler. We expect this gap to narrow with straightforward engineering efforts. 
 
Our main evaluation comprises 12 tasks across 7 method configurations, each repeated over 5 seeds, yielding a total of 420 runs. Reproducing the full set of main results therefore requires approximately 8{,}400 single-GPU hours.

\subsection{Code}
\label{app:code}
The code used to produce the experiments is original code by the authors and based on PyTorch (https://pytorch.org/).

The environments are derived from the following projects:
\begin{enumerate}
    \item OGBench (https://github.com/seohongpark/ogbench) — MIT License
    \item MiniGrid (https://github.com/Farama-Foundation/Minigrid) — Apache License 2.0
\end{enumerate}
\section{Theoretical Analysis}
\label{app:theory}

\subsection{Tightness of the Gaussian Entropy Bound (A3)}
\label{app:entropy_bound}

Approximation~A3 replaces the mixture entropy $H(s_{1:T})$ with the
entropy of the moment-matched Gaussian
$q = \mathcal{N}(\bar{\boldsymbol{\mu}},\, \Sigma)$, where
$\Sigma = \sigma^2 I_{Td} + C$. Because $q$ shares the mean and
covariance of the mixture
$p = \tfrac{1}{M}\sum_{k} \mathcal{N}(\boldsymbol{\mu}_k,\,
\sigma^2 I_{Td})$, the cross-entropy $-\!\int p\log q$ evaluates to
$H(q)$, and the gap introduced by A3 is
\begin{equation}\label{eq:gap_kl}
  H(q) - H(p) \;=\; D_{\mathrm{KL}}(p \,\|\, q) \;\geq\; 0.
\end{equation}
The surrogate therefore overestimates entropy for every trajectory.
We now argue that this overestimation is benign in both regimes
encountered during training.

\paragraph{Well-explored regions.}
As the agent collects data in a region of the state space, the
ensemble members converge toward agreement: the means
$\boldsymbol{\mu}_k$ approach a common function and the epistemic
covariance $C \to 0$. In this limit the mixture $p$ collapses to a
single Gaussian identical to~$q$, so
$D_{\mathrm{KL}}(p \,\|\, q) \to 0$ and the bound becomes tight.
The surrogate is therefore most faithful in precisely the regions
where the agent must discriminate between trajectories of similar
information content.

\paragraph{Unexplored regions.}
When the ensemble members disagree strongly, the mixture is
multi-modal and the Gaussian approximation can be loose. Two
observations limit the practical impact of this looseness. First,
the true information gain is itself large in unexplored regions, so
even a coarse ranking of trajectories directs the agent toward
informative transitions. Second, the policy has little basis for
fine-grained trajectory selection in parts of the state space it has
rarely visited; the surrogate need only distinguish broadly
informative directions from uninformative ones, a task for which the
overestimate suffices. The failure mode of A3 is therefore
over-exploration rather than premature convergence: the agent may
spend additional time in high-uncertainty regions, but it will not
be steered away from them.

\paragraph{Empirical precedent.}
Plan2Explore~\citep{sekarPlanningExploreSelfSupervised2020} relies on the same moment-matching
approximation applied independently to each step: its
ensemble-disagreement reward is the entropy of the Gaussian
matched to the per-step predictive $p(s_t)$, discarding the
mixture structure in exactly the same way as A3. The empirical
success of this per-step surrogate across a range of
continuous-control tasks provides evidence that the Gaussian
approximation is adequate for exploration in practice. Our usage
extends the approximation from the $d$-dimensional per-step
marginal to the $Td$-dimensional joint trajectory distribution,
where cross-step correlations can introduce additional
non-Gaussianity absent from the marginals. The per-step evidence is
therefore encouraging but does not directly certify the
trajectory-level bound; the experiments in
\S\ref{sec:experiments} provide this validation.

\paragraph{Interaction with the trace reduction.}
The bound established above survives the trace reduction (D1) when the epistemic cross-covariance blocks $C_{jt}$ are well summarised by their scalar traces, i.e.\ when the off-diagonal entries of each $d \times d$ block carry negligible mass relative to the diagonal. In this regime each block satisfies $C_{jt} = c_{jt} I_d$, so $\Sigma = B \otimes I_d$ with $B = \tfrac{1}{d}(K + \sigma^2 d\, I_T)$ and
\begin{equation}
    \log\det(C + \sigma^2 I_{Td}) = d\,\log\det(K + \sigma^2 d\, I_T) - dT\log d.
\end{equation}
The factor $d > 0$ and the additive constant $-dT\log d$ are both policy-independent, so $\tilde{J}$ preserves the $\operatorname*{arg\,max}$ over policies and inherits the upper-bound guarantee of A3. As off-diagonal correlations within the blocks grow, D1 discards directional structure and the bounding relationship breaks down (Appendix~\ref{app:capacity_saturation}). The surrogate then derives its justification from the monotonicity and concavity of Proposition~\ref{prop:surrogate_properties} and from the empirical validation in \S\ref{sec:experiments}.
\subsection{Capacity Saturation and the Trace Reduction}
\label{app:capacity_saturation}

The Gaussian bound of Appendix~\ref{app:entropy_bound} reduces the
trajectory-level objective to $\log\det(\Sigma)$ with
$\Sigma = \sigma^2 I_{Td} + C$. This section establishes that the
epistemic covariance~$C$ has insufficient rank to support multi-step
exploration, derives the surrogate obtained without the trace
reduction (used as an ablation baseline in \S\ref{sec:ablations}),
and proves that D1 restores the effective capacity to
$\min(T,\,(M{-}1)d)$.

\begin{theorem}[Capacity saturation and restoration]
\label{thm:saturation}
Let\/
$\boldsymbol{\delta}_k^{(t)} = \mu_k(s_{t-1}, a_{t-1}) -
\bar{\mu}(s_{t-1}, a_{t-1}) \in \mathbb{R}^d$ for\/
$k = 1,\dots,M$ and\/ $t = 1,\dots,T$, and define the epistemic
covariance\/ $C \in \mathbb{R}^{Td \times Td}$ and trace-reduced
kernel\/ $K \in \mathbb{R}^{T \times T}$ as in
Eqs.~\eqref{eq:epistemic_block}--\eqref{eq:gram_matrix}.
\begin{enumerate}[label=(\roman*)]
  \item $\operatorname{rank}(C) \leq M - 1$.
  \item $\operatorname{rank}(K) \leq \min\!\bigl(T,\,(M{-}1)d\bigr)$.
\end{enumerate}
\end{theorem}

\paragraph{Why the full covariance saturates.}
Each $d \times d$ block $C_{jt}$ estimates a cross-covariance with
$d^2$ free parameters from only $M$ outer products. For $M = 5$ and
$d = 256$, this is $65{,}536$ parameters from $5$ samples: the
resulting estimate has rank at most $M - 1$ regardless of the true
rank, and the directional structure it reports is dominated by
sampling noise.

\begin{proof}[Proof of\/ \textup{(i)}]
Stack the per-step deviations into
$\mathbf{d}_k = [(\boldsymbol{\delta}_k^{(1)})^\top,\, \dots,\,
(\boldsymbol{\delta}_k^{(T)})^\top]^\top \in \mathbb{R}^{Td}$. Then
$C = \frac{1}{M}\sum_{k=1}^{M}
\mathbf{d}_k\,\mathbf{d}_k^\top$, a sum of $M$ rank-one matrices
with column spaces contained in
$\operatorname{span}\{\mathbf{d}_1,\dots,\mathbf{d}_M\}$, so
$\operatorname{rank}(C) \leq M$. The centering identity
$\sum_{k=1}^{M}\mathbf{d}_k = \mathbf{0}$ (which holds by
definition of~$\bar{\mu}$) forces the last vector into the span of
the first $M - 1$, giving
$\operatorname{rank}(C) \leq M - 1$.
\end{proof}

The rank bound has a direct information-theoretic reading.
Assumption~A1 reduces exploration to classifying which of $M$
ensemble members generated the trajectory, a problem with at most
$\log M$ bits of capacity. In high-dimensional latent spaces
($d \gg M$), a single observation is typically sufficient to resolve
this classification, leaving no reward signal for subsequent steps
in the rollout.

\paragraph{The surrogate without D1.}
Define
$D = \frac{1}{\sqrt{M}} [\mathbf{d}_1 \mid \cdots \mid
\mathbf{d}_M] \in \mathbb{R}^{Td \times M}$, so that
$C = DD^\top$. The Sylvester determinant identity gives
\begin{equation}\label{eq:sylvester}
  \log\det\!\bigl(\sigma^2 I_{Td} + C\bigr)
  = (Td - M)\log\sigma^2
  + \log\det\!\bigl(\sigma^2 I_M + D^\top\! D\bigr).
\end{equation}
The first term is independent of the policy. The second is the
log-determinant of the $M \times M$ Gram matrix
$G = D^\top\!D$, which has a zero eigenvalue due to centering
($G\,\mathbf{1}_M = \mathbf{0}$) and therefore at most $M - 1$
policy-dependent degrees of freedom. This is the non-trace-reduced
variant of the surrogate; we compare it against CIG in the ablations
of \S\ref{sec:ablations}.

\paragraph{How the trace reduction restores capacity.}
The trace $K_{jt} = \operatorname{tr}(C_{jt})$ replaces each
unreliable $d \times d$ directional estimate with a scalar summary,
the total variance, that can be estimated reliably from $M$ samples.
This rearranges the data: instead of $M$ vectors in
$\mathbb{R}^{Td}$ (rank $\leq M - 1$), the signal is carried by
$T$ rows of a $T \times Md$ factor matrix (rank
$\leq \min(T,\,(M{-}1)d)$), and the binding constraint shifts from
the number of ensemble members to the rollout horizon.

\begin{proof}[Proof of\/ \textup{(ii)}]
For each member~$k$, define the row matrix
$\Delta_k \in \mathbb{R}^{T \times d}$ whose $t$-th row is
$(\boldsymbol{\delta}_k^{(t)})^\top$. Then
$K = \frac{1}{M}\sum_{k=1}^{M}\Delta_k\,\Delta_k^\top$. Define
$F = \frac{1}{\sqrt{M}} [\Delta_1 \mid \cdots \mid \Delta_M]
\in \mathbb{R}^{T \times Md}$, so that $K = FF^\top$ and
$\operatorname{rank}(K) = \operatorname{rank}(F)
\leq \min(T,\, Md)$. The centering constraint
$\sum_k \Delta_k = 0$ forces the last $d$ columns of~$F$ into the
span of the preceding $(M{-}1)d$ columns, so
$\operatorname{rank}(F) \leq \min(T,\,(M{-}1)d)$.
\end{proof}

\paragraph{What the trace reduction preserves and sacrifices.}
The trace reduction sacrifices the bounding relationship to
$I(w;\, s_{1:T})$. To see why, note that the Gaussian bound of
Eq.~\eqref{eq:gaussian_bound} guarantees
$H(s_{1:T}) \leq \tfrac{1}{2}\log\det(2\pi e\,\Sigma)$; preserving
this guarantee after D1 would require
$\log\det(\sigma^2 I_{Td} + C) \leq
\log\det(K + \sigma^2 d \cdot I_T) + \text{const}$
for all policies. No such inequality holds in general: the
block-trace collapse can either increase or decrease the
log-determinant depending on the spectral structure of~$C$.
The surrogate $\tilde{J}$ is therefore neither an upper nor a lower
bound on the Gaussian entropy.

The following proposition establishes that $\tilde{J}$ nonetheless
retains the two structural properties needed for a well-behaved
exploration objective.

\begin{proposition}[Structural properties of the surrogate]
\label{prop:surrogate_properties}
Let\/ $\tilde{J}(K) = \log\det(K + \sigma^2 d \cdot I_T)$
with\/ $K \in \mathcal{S}_+^T$. The ridge ensures\/
$K + \sigma^2 d \cdot I_T \in \mathcal{S}_{++}^T$, so\/ $\tilde{J}$
is well-defined on all of\/ $\mathcal{S}_+^T$.
\begin{enumerate}[label=(\roman*)]
  \item \textbf{Monotonicity.} If\/ $K' \succeq K$ in the
    positive-semidefinite order, then\/
    $\tilde{J}(K') \geq \tilde{J}(K)$.
  \item \textbf{Concavity.} $\tilde{J}$ is concave on\/
    $\mathcal{S}_+^T$.
\end{enumerate}
\end{proposition}

\begin{proof}
For (i): $K' \succeq K$ implies
$K' + \sigma^2 d \cdot I_T \succeq K + \sigma^2 d \cdot I_T \succ 0$,
and $\log\det$ is monotone on
$\mathcal{S}_{++}^T$~\citep[Theorem~7.6.7]{0521386322}.
For (ii): $\log\det$ is concave on
$\mathcal{S}_{++}^T$, and
$K \mapsto K + \sigma^2 d \cdot I_T$ is affine; the composition of a
concave function with an affine map is concave.
\end{proof}

Monotonicity ensures that visiting a state with higher disagreement
cannot decrease the objective. Concavity implies diminishing returns:
the marginal value of the $t$-th step decreases as earlier steps
account for more of the epistemic signal, penalising redundant
rollouts. Together, the two properties replace the lost bounding
guarantee with direct control over the incentive landscape that the
policy optimises.

For the practical values $M = 5$ and $d \geq 200$, D1 raises the
capacity ceiling from~$4$ to a minimum of~$800$, which exceeds any rollout
horizon used in this paper ($T = 15$). The binding constraint on
the rank of~$K$ is therefore~$T$: every step in the rollout can
contribute an independent direction of signal to the log-determinant
surrogate. We isolate the effect of this capacity restoration
empirically in \S\ref{sec:ablations}.
\subsection{Limiting Cases of the CIG Reward}
\label{app:reward_limits}

The interpretation paragraph of \S\ref{sec:reward_derivation} claims
that the CIG reward $r_t$ interpolates between the lifelong
disagreement bonus and the aleatoric floor depending on the
relationship between step~$t$ and the rollout prefix. The following
proposition makes both limits precise.

\begin{proposition}[Limiting cases]
\label{prop:reward_limits}
Let\/ $r_t$ be defined as in Eq.~\eqref{eq:cig_reward}, with
$\mathbf{k}_{<t} \in \mathbb{R}^{t-1}$ collecting the entries
$K_{jt}$ for\/ $j < t$ and\/
$\tilde{K}_{<t} = K_{<t} + \sigma^2 d \cdot I_{t-1}$.
\begin{enumerate}[label=(\roman*)]
  \item \textbf{Orthogonal disagreement.} If\/
    $K_{jt} = 0$ for all\/ $j < t$, then
    $r_t = \log(K_{tt} + \sigma^2 d)$.
  \item \textbf{Redundant disagreement.} If there exists\/
    $\boldsymbol{\alpha} \in \mathbb{R}^{t-1}$ such that\/
    $\mathbf{k}_{<t} = K_{<t}\,\boldsymbol{\alpha}$ and\/
    $K_{tt} = \boldsymbol{\alpha}^\top K_{<t}\,\boldsymbol{\alpha}$
    (i.e., row~$t$ of\/ $K$ lies in the row span of\/ $K_{<t}$),
    then
    \begin{equation}\label{eq:reward_sandwich}
      \sigma^2 d
      \;\leq\;
      \exp(r_t)
      \;\leq\;
      \sigma^2 d\,\bigl(1 + \lVert\boldsymbol{\alpha}\rVert^2\bigr).
    \end{equation}
\end{enumerate}
\end{proposition}

\begin{proof}
For~(i): $\mathbf{k}_{<t} = \mathbf{0}$ implies that the
subtraction in Eq.~\eqref{eq:cig_reward} vanishes, giving
$\exp(r_t) = K_{tt} + \sigma^2 d$.

For~(ii): substituting $\mathbf{k}_{<t} = K_{<t}\,\boldsymbol{\alpha}$
and $K_{tt} = \boldsymbol{\alpha}^\top K_{<t}\,\boldsymbol{\alpha}$
into Eq.~\eqref{eq:cig_reward} and simplifying yields
\begin{equation}\label{eq:schur_redundant}
  \exp(r_t)
  = \sigma^2 d\,\bigl(
    1 + \lVert\boldsymbol{\alpha}\rVert^2
    - \sigma^2 d\;\boldsymbol{\alpha}^\top
      \tilde{K}_{<t}^{-1}\,\boldsymbol{\alpha}
  \bigr).
\end{equation}
The upper bound follows from
$\tilde{K}_{<t} \succeq \sigma^2 d \cdot I_{t-1}$, which gives
$\boldsymbol{\alpha}^\top \tilde{K}_{<t}^{-1}\,\boldsymbol{\alpha}
\geq 0$. The lower bound follows from
$\tilde{K}_{<t}^{-1} \preceq (\sigma^2 d)^{-1} I_{t-1}$, which gives
$\sigma^2 d\;\boldsymbol{\alpha}^\top
\tilde{K}_{<t}^{-1}\,\boldsymbol{\alpha}
\leq \lVert\boldsymbol{\alpha}\rVert^2$.
\end{proof}

Part~(i) recovers the lifelong disagreement reward: with no
correlation to the prefix, the episodic correction is inactive and
$r_t$ equals the total disagreement at step~$t$, the same signal
used by Plan2Explore~\citep{sekarPlanningExploreSelfSupervised2020}. Part~(ii) shows that when
step~$t$ is fully redundant with the prefix, the epistemic
contribution $K_{tt}$ is removed and $\exp(r_t)$ is confined to a
band proportional to the aleatoric scale~$\sigma^2 d$. In the common
case of exact duplication ($\boldsymbol{\alpha} = \mathbf{e}_j$ for
some $j < t$), the band tightens to
$[\sigma^2 d,\; 2\sigma^2 d]$, confirming that repeated visits to
the same model gap yield a reward near the aleatoric floor
regardless of the magnitude of the disagreement.
\subsection{Characterisation of the Prefix-Redundancy Term}
\label{app:prefix_redundancy}

The prefix-redundancy correction $\mathbf{k}_{<t}^\top \tilde{K}_{<t}^{-1} \mathbf{k}_{<t}$ in Eq.~\eqref{eq:cig_reward} distinguishes CIG from per-step ensemble disagreement. Removing this term recovers the \emph{No Prefix Redundancy} ablation of \S\ref{sec:experiments}, whose reward reduces to $\log(K_{tt} + \sigma^2 d)$. This subsection asks two questions: (i)~where in the state space does the correction activate, and (ii)~how does that activation pattern evolve over training? Figure~\ref{fig:prefix_heatmaps} summarises the answers on AntMaze across five training windows.

\begin{figure}[t]
    \centering
    \includegraphics[width=\linewidth]{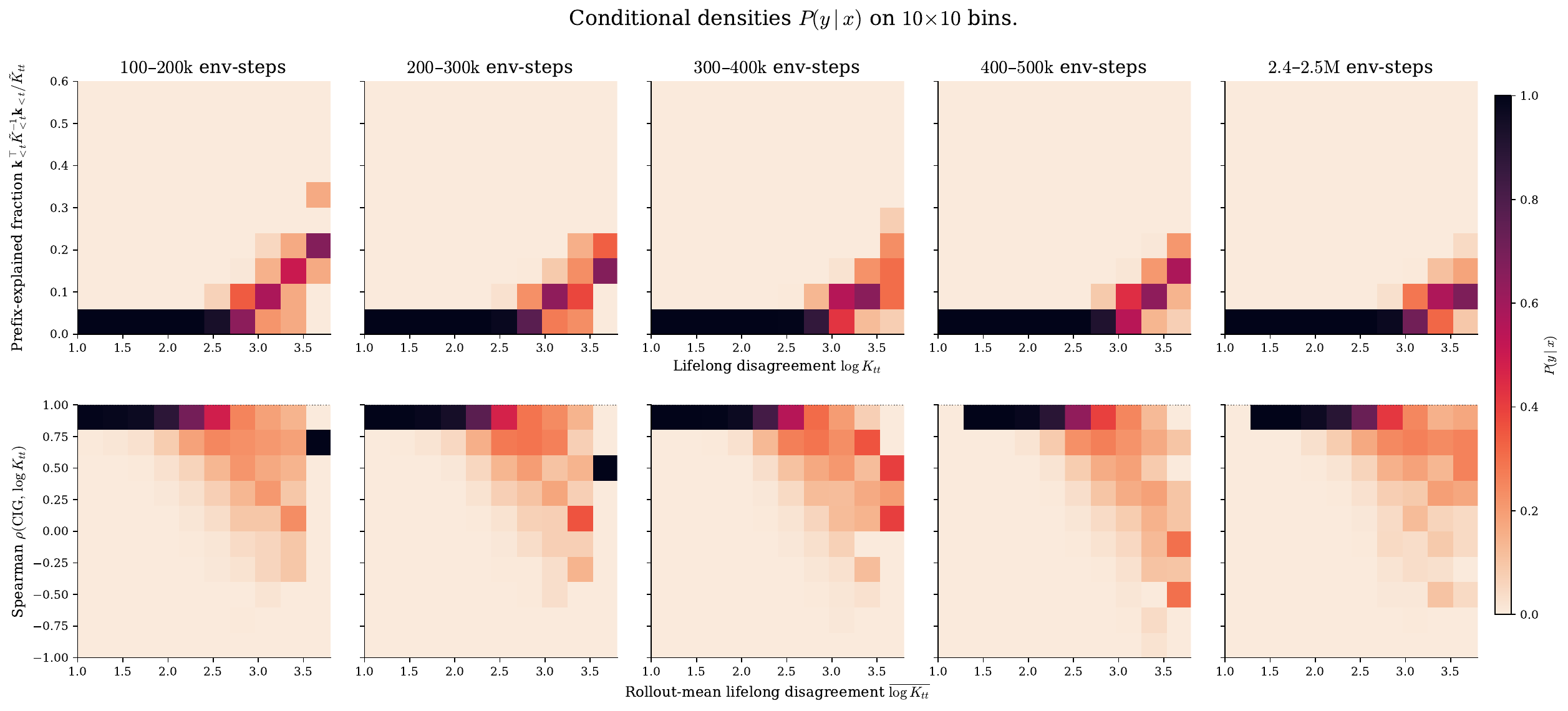}
    \caption{\textbf{Prefix-redundancy activation across training (AntMaze-medium).} Each panel shows the column-normalised conditional density $P(y \mid x_{\text{bin}})$ on a $10 \times 10$ grid. \textbf{Top row:} prefix-explained fraction $\mathbf{k}_{<t}^\top \tilde{K}_{<t}^{-1} \mathbf{k}_{<t} / K_{tt}$ (the share of the lifelong disagreement at step~$t$ accounted for by the prefix) vs.\ lifelong disagreement $\log K_{tt}$, the diagonal term of Eq.~\eqref{eq:cig_reward}. \textbf{Bottom row:} Spearman rank correlation $\rho(\mathrm{CIG},\,\log K_{tt})$ between the CIG reward and the lifelong disagreement within each rollout vs.\ rollout-mean lifelong disagreement $\overline{\log K_{tt}}$. Under the \emph{No Prefix Redundancy} ablation the top-row quantity is identically zero and the bottom-row correlation is identically one; departures from these values reflect off-diagonal structure in $K$ that CIG exploits.}
    \label{fig:prefix_heatmaps}
\end{figure}

\paragraph{Frontier selectivity at the step level.}
The top row of Figure~\ref{fig:prefix_heatmaps} plots the prefix-explained fraction $\mathbf{k}_{<t}^\top \tilde{K}_{<t}^{-1} \mathbf{k}_{<t} / K_{tt}$ against the lifelong disagreement $\log K_{tt}$, the diagonal entry of the kernel that constitutes the first term of Eq.~\eqref{eq:cig_reward}. The correction concentrates in the rightmost bins, where lifelong disagreement is high, and collapses to near zero for low-$K_{tt}$ steps. The mechanism is straightforward: the Cauchy--Schwarz inequality on the Gram matrix forces $|K_{jt}| \leq \sqrt{K_{jj}\, K_{tt}}$, so when both a prefix step $j$ and the current step $t$ lie in well-learned regions, the off-diagonal entries that feed the quadratic form are bounded by the geometric mean of two small quantities. The prefix-redundancy term therefore activates only when step~$t$ sits at the exploration frontier, where $K_{tt}$ is large enough for the off-diagonals to carry non-trivial alignment with the prefix.

This selectivity sharpens as training progresses. Over successive windows the conditional mass migrates upward and to the right: the prefix accounts for 10--25\% of the lifelong disagreement at frontier steps by 2.4--2.5M env-steps, while low-disagreement steps remain unaffected. The pattern is consistent with the ensemble concentrating its residual uncertainty on a narrowing frontier whose steps become increasingly correlated with one another.

\paragraph{Rank reordering at the rollout level.}
The bottom row of Figure~\ref{fig:prefix_heatmaps} examines the same effect at rollout granularity, plotting $\rho(\mathrm{CIG},\,\log K_{tt})$, the Spearman rank correlation between CIG and the lifelong disagreement $\log K_{tt}$ within each rollout, against rollout-mean lifelong disagreement $\overline{\log K_{tt}}$. Because the ridge $\sigma^2 d$ is constant across steps, $\log K_{tt}$ and $\log(K_{tt} + \sigma^2 d)$ are related by a strictly monotone transform; Spearman $\rho$ is invariant to such transforms, so the correlation is identical whether the reference ranking uses the \emph{Lifelong Only} variant ($\log K_{tt}$, which removes both the prefix correction and the ridge) or the \emph{No Prefix Redundancy} variant ($\log(K_{tt} + \sigma^2 d)$, which retains the ridge but drops the prefix correction).

When rollout-mean lifelong disagreement is low, $\rho$ clusters near one: the imagined trajectory stays in well-learned regions, the prefix-redundancy term is inactive, and CIG ranks steps identically to the lifelong disagreement alone. As $\overline{\log K_{tt}}$ increases, $\rho$ spreads across a wide range, with substantial mass at $\rho < 0.5$ and occasional rank reversals ($\rho < 0$). CIG reorders steps within a rollout only when the rollout visits the frontier, leaving the ranking unchanged in familiar territory. The effect grows over training, consistent with the narrowing frontier producing stronger inter-step correlations that the off-diagonal entries of $K$ capture.

\paragraph{Contrast with E3B$\times$P2E.}
CIG's frontier selectivity is a structural consequence of deriving both the lifelong and prefix terms from the same kernel $K$. The Cauchy--Schwarz coupling $|K_{jt}| \leq \sqrt{K_{jj}\, K_{tt}}$ ties the magnitude of the correction to the diagonal, forcing it to vanish quadratically when $K_{tt}$ is small. E3B$\times$P2E lacks this coupling. Its E3B factor scores each step against an embedding-space covariance $C_t^{-1}$ built from the encoder $\phi$, not from the ensemble's epistemic structure. In well-learned corridors where two states happen to have distant embeddings under $\phi$, E3B assigns a large novelty bonus; suppression relies entirely on the P2E factor being small, an indirect cancellation rather than a structural zero.

\paragraph{Formal statement.}
The quadratic vanishing observed in Figure~\ref{fig:prefix_heatmaps} can be made precise. Define the signal-to-noise ratio at step $t$ as $\epsilon_t = K_{tt} / (\sigma^2 d)$ and let $\epsilon = \max_{1 \leq t \leq T} \epsilon_t$. Write $r_t^{\mathrm{diag}} = \log(K_{tt} + \sigma^2 d)$ for the \emph{No Prefix Redundancy} reward.

\begin{proposition}[Frontier inactivity]\label{prop:frontier_inactivity}
For all $t \leq T$,
\begin{equation}\label{eq:frontier_inactivity}
    0 \;\leq\; r_t^{\mathrm{diag}} - r_t \;\leq\; -\log\!\bigl(1 - (t{-}1)\,\epsilon^2\bigr).
\end{equation}
When $(t{-}1)\,\epsilon^2 \leq 1/2$, this simplifies to $r_t^{\mathrm{diag}} - r_t \leq 2(t{-}1)\,\epsilon^2$.
\end{proposition}

\begin{proof}
\emph{Lower bound.} Because $K$ is a Gram matrix ($K = \frac{1}{M}\Delta^\top \Delta$, where column $t$ of $\Delta \in \mathbb{R}^{Md \times T}$ concatenates the deviation vectors $\boldsymbol{\delta}_k^{(t)}$ across members), $K$ is positive semidefinite. Hence $\tilde{K}_{<t} = K_{<t} + \sigma^2 d \cdot I_{t-1} \succ 0$ and $\mathbf{k}_{<t}^\top \tilde{K}_{<t}^{-1} \mathbf{k}_{<t} \geq 0$, giving $r_t \leq r_t^{\mathrm{diag}}$.

\emph{Upper bound.} Since $\tilde{K}_{<t} \succeq \sigma^2 d \cdot I_{t-1}$, we have $\tilde{K}_{<t}^{-1} \preceq (\sigma^2 d)^{-1}\, I_{t-1}$, so
\begin{equation}\label{eq:quad_form_bound}
    \mathbf{k}_{<t}^\top \tilde{K}_{<t}^{-1} \mathbf{k}_{<t} \;\leq\; \frac{\|\mathbf{k}_{<t}\|^2}{\sigma^2 d}.
\end{equation}
Cauchy--Schwarz on the Gram matrix gives $K_{jt}^2 \leq K_{jj}\, K_{tt}$ for all $j, t$. Summing over the prefix indices and substituting $K_{jj} \leq \epsilon\, \sigma^2 d$:
\begin{equation}\label{eq:norm_bound}
    \|\mathbf{k}_{<t}\|^2 = \sum_{j=1}^{t-1} K_{jt}^2 \;\leq\; K_{tt} \sum_{j=1}^{t-1} K_{jj} \;\leq\; (t{-}1)\,\epsilon\,\sigma^2 d \cdot K_{tt}.
\end{equation}
Combining Eqs.~\eqref{eq:quad_form_bound} and~\eqref{eq:norm_bound}:
\begin{equation}
    \frac{\mathbf{k}_{<t}^\top \tilde{K}_{<t}^{-1} \mathbf{k}_{<t}}{K_{tt} + \sigma^2 d} \;\leq\; \frac{(t{-}1)\,\epsilon\, K_{tt}}{K_{tt} + \sigma^2 d} \;=\; (t{-}1)\,\epsilon\,\frac{\epsilon_t}{1 + \epsilon_t} \;\leq\; (t{-}1)\,\epsilon^2.
\end{equation}
Writing $r_t^{\mathrm{diag}} - r_t = -\log\!\bigl(1 - \mathbf{k}_{<t}^\top \tilde{K}_{<t}^{-1} \mathbf{k}_{<t} / (K_{tt} + \sigma^2 d)\bigr)$ and applying monotonicity of $-\log(1-x)$ yields Eq.~\eqref{eq:frontier_inactivity}. The simplified bound uses $-\log(1-x) \leq 2x$ for $x \in [0,\,1/2]$.
\end{proof}

\section{Extended Experimental Results}
\label{app:extended_results}

\subsection{Baselines}
\label{sec:baselines}

This section briefly recaps the intrinsic-reward baselines used as
comparisons in our experiments. Each paragraph below summarises only
the per-step \emph{reward signal} that defines the method, since this
is what differs across baselines and what is directly comparable to
CIG; encoder, ensemble, and optimisation details follow the cited
papers, and the numerical settings we used are reported in
\S\ref{sec:hparams}.

\paragraph{ICM.}
Intrinsic Curiosity Module learns a state encoder $\phi$ jointly with
an inverse model that predicts the action from
$(\phi(s_t), \phi(s_{t+1}))$, so that $\phi$ retains only the part of
the observation that the agent's actions can affect. A forward model
$\hat{f}$ is then trained to predict the next-state encoding from
$(\phi(s_t), a_t)$, and its squared prediction error is taken as the
intrinsic reward:
\begin{equation}
r_t^{\text{ICM}} = \tfrac{1}{2}\,\big\lVert \hat{f}(\phi(s_t), a_t) - \phi(s_{t+1}) \big\rVert_2^2.
\end{equation}
The inverse-model objective is what gives ICM its nominal robustness
to action-irrelevant pixel noise, but the bonus itself is still a raw
forward prediction error and so remains susceptible to genuine
stochasticity in the transitions~\citep{pathakCuriosityDrivenExplorationSelfSupervised2017}.

\paragraph{RND.}
Random Network Distillation defines novelty against a randomly
initialised, \emph{fixed} target network $f^{*}$. A predictor network
$\hat{f}$ is trained to regress $f^{*}$ on visited states only, and
the residual error is the intrinsic reward:
\begin{equation}
r_t^{\text{RND}} = \big\lVert \hat{f}(s_t) - f^{*}(s_t) \big\rVert_2^2.
\end{equation}
States the predictor has been trained on yield small error, while
under-visited states retain large error, so the bonus operates as a
state-level novelty signal that is independent of the agent's actions
and of any learned dynamics model~\citep{burdaExplorationRandomNetwork2019}.

\paragraph{Plan2Explore (Disagreement).}
Plan2Explore equips a Dreamer world model with an ensemble of $M$
one-step latent predictors $\{f_k\}_{k=1}^{M}$ that all predict the
next-state embedding from the current latent and action. The
intrinsic reward is the average predictive variance across the
ensemble, taken over the embedding's $d$ dimensions:
\begin{equation}
r_t^{\text{P2E}} = \tfrac{1}{d}\sum_{i=1}^{d} \mathrm{Var}_{k}\big[f_k(s_t, a_t)_i\big].
\end{equation}
Because the ensemble members agree on aleatoric noise but disagree
where the world model has not yet generalised, this variance
approximates epistemic uncertainty over the next
embedding~\citep{sekarPlanningExploreSelfSupervised2020, pathakSelfSupervisedExplorationDisagreement2019}. CIG re-uses
exactly this ensemble; the difference is that P2E reads off only its
\emph{marginal} disagreement at each step, whereas CIG conditions on
prior steps within the rollout.

\paragraph{E3B.}
Elliptical Episodic Bonus (E3B) reuses ICM's inverse-dynamics
objective to train an encoder $\phi$, but its bonus is not a
prediction error. Within a single episode it maintains a running
inverse covariance $C_t^{-1} \in \mathbb{R}^{e\times e}$ of the
visited embeddings with $e$ being the embedding dimensionality, initialised to $(1/\lambda)\,I$, and scores each
step with the elliptical (Mahalanobis-style) bonus
\begin{equation}
r_t^{\text{E3B}} \;=\; \phi(s_t)^{\!\top}\, C_t^{-1}\, \phi(s_t),
\end{equation}
which is large when $\phi(s_t)$ points in a direction the current
episode has not yet covered. After each step $C_t^{-1}$ is updated by
a Sherman--Morrison rank-1 update and is reset at episode boundaries.
Unlike ICM and RND, novelty is therefore computed \emph{within} the
current episode rather than against a global
predictor~\citep{henaffExplorationEllipticalEpisodic2022}.

\paragraph{E3B $\times$ P2E.}
This baseline composes E3B's per-episode elliptical bonus with
Plan2Explore's lifelong ensemble disagreement multiplicatively,
\begin{equation}
r_t^{\text{E3B}\times\text{P2E}} \;=\; r_t^{\text{E3B}} \cdot r_t^{\text{P2E}},
\end{equation}
so the agent is rewarded only where \emph{both} signals fire: the
state must be novel relative to what this episode has already seen
(E3B) \emph{and} lie in a region the world-model ensemble disagrees
about (P2E). Following~\citet{henaffExplorationEllipticalEpisodic2022}, the encoder embedding is
$\ell_2$-normalised before entering the elliptical update so that the
per-step E3B bonus is bounded and the multiplicative product is
well-scaled. This combination serves as a simple non-CIG baseline
that already mixes episodic novelty with epistemic uncertainty.

\paragraph{APT.}
Active Pre-Training estimates the entropy of visited states
non-parametrically with a particle-based $k$-nearest-neighbour
estimator on the deterministic world-model state. For a particle
$h_t$ in the current batch, the reward is the log of the average
distance to its $k$ nearest neighbours $\mathcal{N}_k(h_t)$, which
approximates state entropy up to an additive constant:
\begin{equation}
r_t^{\text{APT}} = \log\!\left(1 + \tfrac{1}{k}\sum_{h_j \in \mathcal{N}_k(h_t)} \lVert h_t - h_j \rVert_2 \right).
\end{equation}
Unlike the other baselines, the APT bonus directly rewards visiting points that are far from the rest
of the batch in feature space, so the agent is pushed towards
under-represented regions of the state distribution rather than
towards regions where some predictor is currently
inaccurate~\citep{liuBehaviorVoidUnsupervised2021}.

\subsection{Per-Episode Entropy and Downstream Success on Minigrid}
\label{app:entropy_success}

\begin{figure}[t]
\centering
\includegraphics[width=\linewidth]{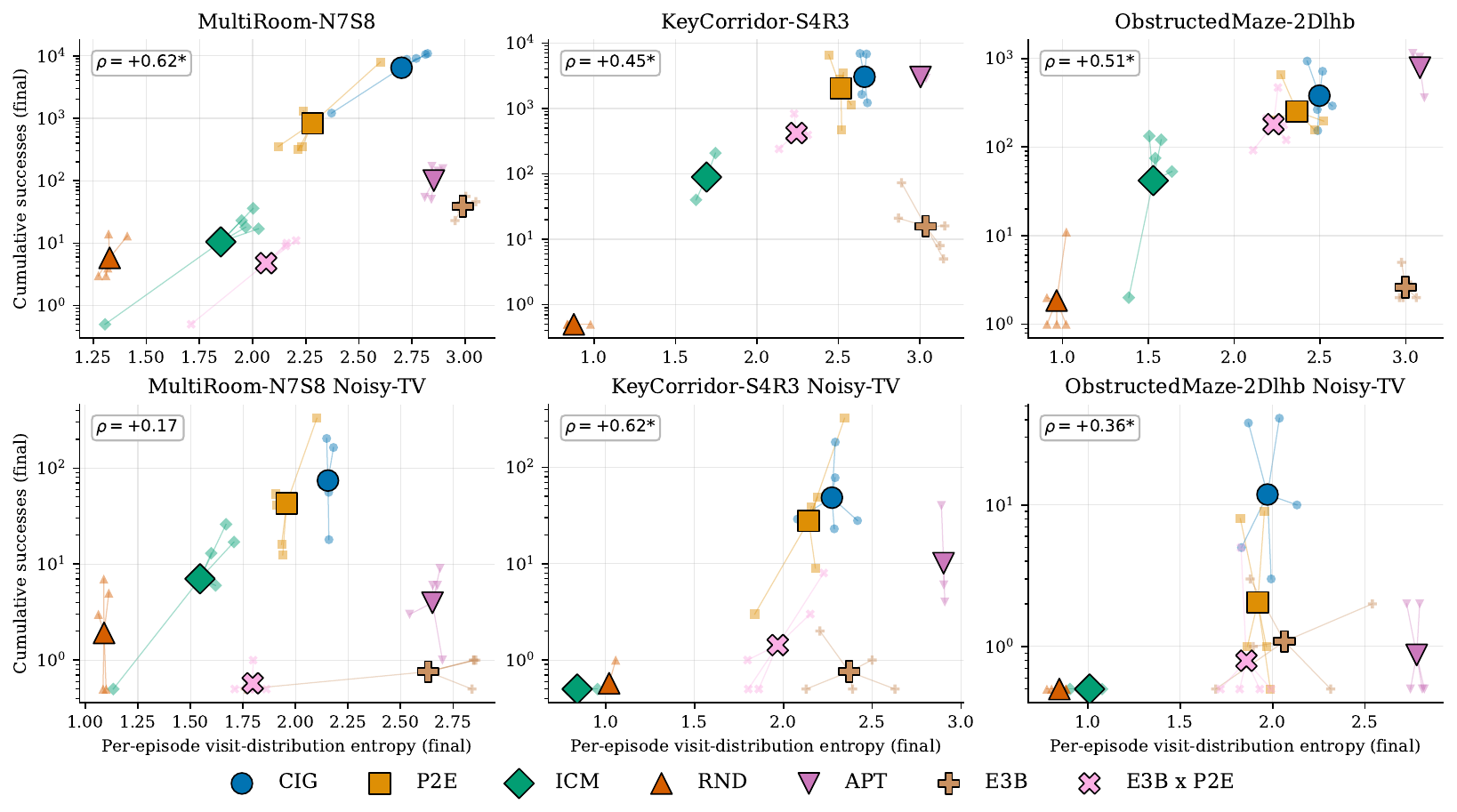}
\caption{%
Per-episode visit-distribution entropy (x-axis) versus final cumulative successes (y-axis, log scale) for six methods across three MiniGrid tasks in clean (top) and Noisy-TV variants (bottom). Large markers denote per-method means; small markers show individual seeds. Each panel also reports the Spearman rank correlation $\rho$.
}
\label{fig:entropy_vs_success}
\end{figure}

The sample-efficiency results in \S\ref{sec:exp_minigrid} leave open whether CIG's advantage arises from visiting more diverse states per episode or from directing visits toward epistemically informative regions. To disentangle these factors, we compute the entropy of the empirical visit-distribution over grid cells for each episode, average over episodes in the final 10\% of training, and plot the result against final cumulative successes in Figure~\ref{fig:entropy_vs_success}.

\paragraph{Per-method entropy ordering.}
Across all six panels, a consistent ordering emerges. E3B and APT achieve the highest per-episode entropy: E3B's episodic novelty bonus pushes the agent to visit cells it has not yet seen within the current episode, while APT's $k$-nearest-neighbour entropy estimate directly rewards uniform state coverage. CIG produces moderately higher entropy than P2E on every task, consistent with the log-determinant objective penalising within-rollout redundancy and thereby encouraging the policy to spread visits across distinct regions. RND and ICM cluster at lower entropy, reflecting their tendency to concentrate on narrow high-novelty frontiers rather than distributing visits broadly.

\paragraph{Entropy does not predict success.}
Despite their consistently high entropy, E3B and APT occupy opposite ends of the success axis depending on the task. E3B achieves the highest or near-highest entropy on five of six panels yet accumulates among the lowest cumulative successes everywhere, including the clean tasks. Its episodic novelty signal distributes visits uniformly but does not prioritise states where the dynamics model carries residual error, so broad coverage does not translate into informative data collection. APT follows a different pattern: its high entropy translates into strong performance on clean ObstructedMaze~2DlhB, where the task requires reaching many spatially distributed objects and uniform coverage aligns well with the downstream metric. On the remaining five panels, APT's entropy advantage yields no corresponding success advantage. Under distraction (bottom row), the dissociation is stark: APT retains the highest entropy on Drastic-NoisyTV ObstructedMaze (mean $\approx 2.75$) but accumulates near-zero successes, because the additional coverage is consumed by distractor states rather than task-relevant regions.

\paragraph{Targeted exploration as the driving factor.}
CIG's entropy is lower than APT's and E3B's on every task, yet CIG achieves the highest cumulative successes on five of six panels and remains competitive on the sixth (clean ObstructedMaze). The gap between CIG and P2E is instructive: CIG produces only modestly higher entropy than P2E, so the performance difference between them cannot be attributed to broader coverage alone. Instead, the prefix correction in Eq.~\eqref{eq:cig_reward} redirects visits away from uncertainty directions already probed earlier in the rollout, converting a small entropy increase into a larger gain in the diversity of epistemic information collected per episode. The pattern across all six panels supports the same conclusion: in MiniGrid, targeted reduction of dynamics-model uncertainty dominates raw state-coverage breadth as a predictor of downstream task success.
\subsection{Exploration selectivity on Puzzle \texorpdfstring{$3{\times}3$}{3x3}}
\label{app:puzzle_flips}
 
\begin{figure}[h]
\centering
\includegraphics[width=0.7\linewidth]{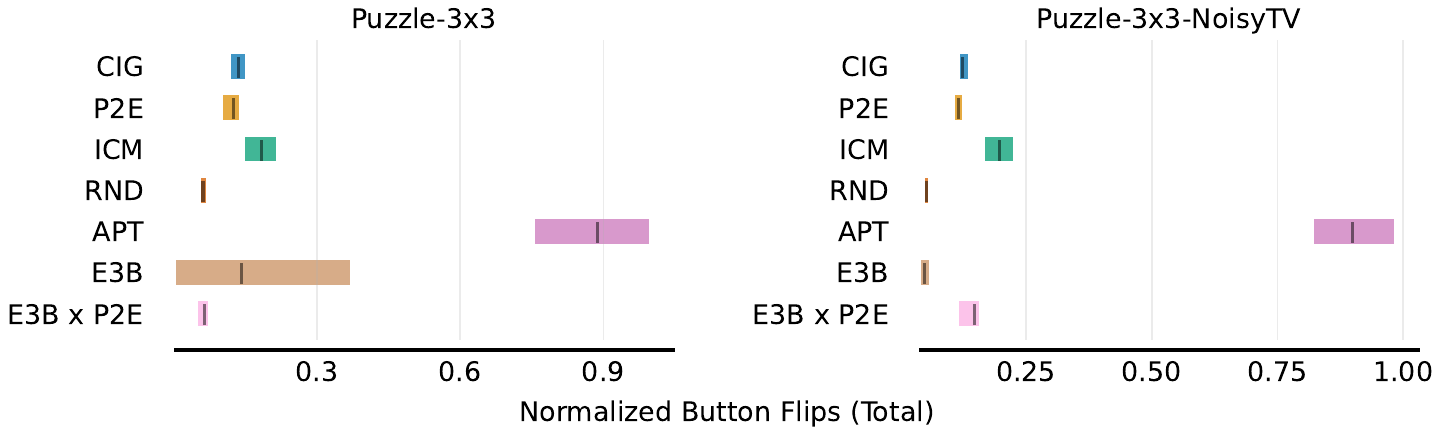}
\caption{Total button flips over training on Puzzle~$3{\times}3$ (clean). CIG, APT, and ICM reach comparable final unique-configuration coverage (Figure~\ref{fig:continuous}, top center), but APT accumulates roughly an order of magnitude more button flips than CIG.}
\label{fig:puzzle_flips}
\end{figure}
 
Figure~\ref{fig:puzzle_flips} reports total button flips for each method on clean Puzzle~$3{\times}3$. While CIG, APT, and ICM achieve similar final unique-configuration coverage (Figure~\ref{fig:continuous}), APT accumulates roughly $10{\times}$ more flips over the same step budget. APT's $k$-nearest-neighbour entropy proxy rewards visiting any under-represented state, driving the policy to press buttons indiscriminately; many flips revisit previously discovered configurations. CIG concentrates interactions on transitions whose ensemble disagreement is orthogonal to the rollout prefix, yielding new configurations at a higher rate per flip.

\section{Related Work}\label{sec:related_work}

Exploration under sparse or absent extrinsic reward relies almost exclusively on intrinsic motivation, where the agent constructs its own reward from model uncertainty, state novelty, or information gain~\citep{schmidhuber_formal_1991, 10.3389/neuro.12.006.2007}. The resulting design space is organised by what context the reward at step~$t$ conditions on: lifetime experience~$\mathcal{D}$, the trajectory prefix~$s_{<t}$, or both (§\ref{sec:bg_proxies}). Existing methods treat these contexts as alternatives or combine them heuristically; CIG derives a single reward that conditions on both.

\paragraph{Lifelong intrinsic rewards.}
Count-based methods generalise visitation counts to continuous spaces via pseudo-counts or density models~\citep{bellemareUnifyingCountBasedExploration2016, ostrovskiCountBasedExplorationNeural2017, lobelFlippingCoinsEstimate2023, yangExplorationAntiExplorationDistributional2024}. Prediction-error methods reward novelty through forward-model discrepancies: ICM~\citep{pathakCuriosityDrivenExplorationSelfSupervised2017} learns an inverse-dynamics model to obtain representations free of action-independent noise, and AMA~\citep{mavor-parkerHowStayCurious2022a} extends this to also filter action-dependent stochasticity. RND~\citep{burdaExplorationRandomNetwork2019} avoids the noisy-TV problem differently by distilling a fixed random network. Other lifelong approaches maximise state entropy~\citep{seoStateEntropyMaximization2021, liuBehaviorVoidUnsupervised2021}, use metric-based bonuses~\citep{wangRethinkingExplorationReinforcement2024}, or reward information gain about dynamics parameters~\citep{houthooftVIMEVariationalInformation2016}. Ensemble-based methods approximate epistemic uncertainty through forward-model disagreement~\citep{pathakSelfSupervisedExplorationDisagreement2019, shyamModelBasedActiveExploration2019, mazzagliaCuriosityDrivenExplorationLatent2022}. All these methods assign rewards per state independently, ignoring correlations within a trajectory (§\ref{sec:method}).

\paragraph{Episodic and hybrid intrinsic rewards.}
Episodic methods define novelty relative to states visited within the current episode~\citep{savinovEpisodicCuriosityReachability2019, joLECOLearnableEpisodic2022a, jiangEpisodicNoveltyTemporal2025}. E3B~\citep{henaffExplorationEllipticalEpisodic2022} is closest to CIG among these: it builds a covariance matrix over inverse-dynamics embeddings using Sherman--Morrison updates, so that each step's reward reflects how much the current state extends the subspace of earlier states. CIG's Cholesky factorisation of~$\mathbf{K}$ plays an analogous role, but E3B is purely episodic: its covariance is built from inverse-dynamics features and reset every episode, whereas CIG conditions on lifetime experience through the ensemble parameters, with off-diagonal entries correcting for within-rollout redundancy given that context. Since episodic bonuses alone ignore lifetime progress, several methods multiply or gate an episodic indicator by a lifelong bonus~\citep{badiaNeverGiveLearning2020, badiaAgent57OutperformingAtari2020, zhangNovelDSimpleEffective2021, raileanuRIDERewardingImpactDriven2020, fuGoImaginationMaximizing2023, zhaRankEpisodesSimple2021, zhangMADEExplorationMaximizing2021}, and empirical studies confirm that the combination improves over either component alone~\citep{henaffStudyGlobalEpisodic2023, jiangImportanceExplorationGeneralization2023}. The episodic components in these methods typically rely on discretised matching (hashing or binned embeddings) that degrades in continuous control domains where states rarely recur in the same bin. More fundamentally, all such hybrids set their combination weights heuristically, whereas CIG arrives at a hybrid signal as a consequence of a single trajectory-level information-gain objective.

\paragraph{Bayesian optimal experimental design.}
CIG's log-determinant surrogate (Eq.~\ref{eq:approx_objective}) instantiates D-optimal design~\citep{doi:10.1137/1.9780898719109} for deep ensembles. In the Bayesian Gaussian setting, D-optimality coincides with maximising expected information gain~\citep{10.1214/aoms/1177728069, 10.1214/ss/1177009939, alexanderian2023briefnotebayesiandoptimality}, the structural fact underlying our passage from the mutual information of Eq.~\eqref{eq:chain_rule} to the Gaussian bound of Eq.~\eqref{eq:gaussian_bound}. \citet{10.1214/23-STS915} survey computational advances in sequential Bayesian experimental design, including the chain-rule decomposition of EIG that parallels our Eq.~\eqref{eq:chain_rule}. Within RL, \citet{mehta2022an} framed model-based exploration as a BOED problem under GP dynamics. TIP~\citep{mehtaExplorationPlanningInformation2022} extended this to trajectory-level joint information gain and is the closest antecedent to CIG in motivation.

\paragraph{Trajectory-level information gain in RL.}
TIP computes joint EIG via log-determinants of GP covariance matrices, naturally penalising redundant state visits, but its reliance on Gaussian-process dynamics restricts it to low-dimensional MDPs. CIG arrives at an analogous objective through ensemble disagreement vectors in a learned latent space, avoiding GP inference entirely; the trace reduction (§\ref{sec:surrogate}) further addresses capacity saturation from finite ensemble size, a challenge absent from TIP's exact GP formulation. Plan2Explore~\citep{sekarPlanningExploreSelfSupervised2020} also uses ensemble disagreement in a learned world model but assigns per-step variance rewards, treating each step independently; \citet{caron2025on} provide theoretical guarantees for such per-step information-gain bonuses but likewise do not model within-trajectory correlations. CIG and Plan2Explore share an ensemble architecture but differ in objective: per-step variance versus trajectory-level information gain (Eq.~\ref{eq:approx_objective}), with the off-diagonal entries of~$\mathbf{K}$ capturing correlations that a per-step reward cannot represent. No existing method derives a per-step intrinsic reward that jointly conditions on lifetime experience and the trajectory prefix from a single information-theoretic objective; §\ref{sec:method} develops the CIG surrogate that achieves this.

\newpage
\section*{NeurIPS Paper Checklist}

\begin{enumerate}

\item {\bf Claims}
    \item[] Question: Do the main claims made in the abstract and introduction accurately reflect the paper's contributions and scope?
    \item[] Answer: \answerYes{} 
    \item[] Justification: We claim that we present a new intrinsic reward that performs better than the baselines on many environments. Our methods section derives this reward and our experiments show that indeed it performs much better on the selected environments and metrics.

\item {\bf Limitations}
    \item[] Question: Does the paper discuss the limitations of the work performed by the authors?
    \item[] Answer: \answerYes{} 
    \item[] Justification: We discuss the limitations shortly in the conclusion section. Since CIG shares its architecture with Plan2Explore, but differs in the reward computation, we did only discuss the limitations that are unique to CIG and not general to ensemble-based intrinsic rewards. The main limitation is as stated scalability to longer horizons. Additionally, we clearly state the assumptions that have made in the methods section and the direct consequences of these assumptions are discussed there and in the appendix.

\item {\bf Theory assumptions and proofs}
    \item[] Question: For each theoretical result, does the paper provide the full set of assumptions and a complete (and correct) proof?
    \item[] Answer: \answerYes{} 
    \item[] Justification: The paper is not a theory paper such that the derivations built upon well-known bounds and identities. The appendix also discusses some of the implications of the made assumptions and states the proof for the implications.

    \item {\bf Experimental result reproducibility}
    \item[] Question: Does the paper fully disclose all the information needed to reproduce the main experimental results of the paper to the extent that it affects the main claims and/or conclusions of the paper (regardless of whether the code and data are provided or not)?
    \item[] Answer: \answerYes{} 
    \item[] Justification: The used hyper parameters and used environments are described in the appendix. We also have pseudo-code in the appendix.

\item {\bf Open access to data and code}
    \item[] Question: Does the paper provide open access to the data and code, with sufficient instructions to faithfully reproduce the main experimental results, as described in supplemental material?
    \item[] Answer: \answerYes{} 
    \item[] Justification: We deliver the code to reproduce the results with detailed instructions how to setup the environment and run the code. 

\item {\bf Experimental setting/details}
    \item[] Question: Does the paper specify all the training and test details (e.g., data splits, hyperparameters, how they were chosen, type of optimizer) necessary to understand the results?
    \item[] Answer: \answerYes{} 
    \item[] Justification: We state the we use 5 seeds for all of our experiments. We describe the environments and hyper parameters in the appendix. Additionally, there is code that can be used to trace the method. Our code mostly follows DreamerV2, a well-known model-based RL agent which additionally lessens the burden of understanding how the results are produced.

\item {\bf Experiment statistical significance}
    \item[] Question: Does the paper report error bars suitably and correctly defined or other appropriate information about the statistical significance of the experiments?
    \item[] Answer: \answerYes{} 
    \item[] Justification: We follow the recommendations of \cite{agarwalDeepReinforcementLearning2021} and report IQM with stratified-bootstrap confidence intervals, the mean and median of the results. This is the community accepted way to get statistical significant results in RL.

\item {\bf Experiments compute resources}
    \item[] Question: For each experiment, does the paper provide sufficient information on the computer resources (type of compute workers, memory, time of execution) needed to reproduce the experiments?
    \item[] Answer: \answerYes{} 
    \item[] Justification: In the appendix we describe the machines that have been used to run the experiments and their typical runtime.
    
\item {\bf Code of ethics}
    \item[] Question: Does the research conducted in the paper conform, in every respect, with the NeurIPS Code of Ethics \url{https://neurips.cc/public/EthicsGuidelines}?
    \item[] Answer: \answerYes{} 
    \item[] Justification: -

\item {\bf Broader impacts}
    \item[] Question: Does the paper discuss both potential positive societal impacts and negative societal impacts of the work performed?
    \item[] Answer: \answerNA{} 
    \item[] Justification: This work is foundational research on intrinsic reward mechanisms for reinforcement learning and is not tied to a specific application or deployment. We do not foresee a direct path from improved intrinsic motivation methods to particular positive or negative societal outcomes beyond those already associated with general advances in reinforcement learning.
    
\item {\bf Safeguards}
    \item[] Question: Does the paper describe safeguards that have been put in place for responsible release of data or models that have a high risk for misuse (e.g., pre-trained language models, image generators, or scraped datasets)?
    \item[] Answer: \answerNA{} 
    \item[] Justification: The research is basic research in exploration in reinforcement learning. There is no immediate way to misuse the models or data.

\item {\bf Licenses for existing assets}
    \item[] Question: Are the creators or original owners of assets (e.g., code, data, models), used in the paper, properly credited and are the license and terms of use explicitly mentioned and properly respected?
    \item[] Answer: \answerYes{} 
    \item[] Justification: The original owners are cited in \ref{app:code}

\item {\bf New assets}
    \item[] Question: Are new assets introduced in the paper well documented and is the documentation provided alongside the assets?
    \item[] Answer: \answerYes{} 
    \item[] Justification: The derived environments are clearly documented in the codebase.

\item {\bf Crowdsourcing and research with human subjects}
    \item[] Question: For crowdsourcing experiments and research with human subjects, does the paper include the full text of instructions given to participants and screenshots, if applicable, as well as details about compensation (if any)? 
    \item[] Answer: \answerNA{} 
    \item[] Justification: -

\item {\bf Institutional review board (IRB) approvals or equivalent for research with human subjects}
    \item[] Question: Does the paper describe potential risks incurred by study participants, whether such risks were disclosed to the subjects, and whether Institutional Review Board (IRB) approvals (or an equivalent approval/review based on the requirements of your country or institution) were obtained?
    \item[] Answer: \answerNA{} 
    \item[] Justification: -

\item {\bf Declaration of LLM usage}
    \item[] Question: Does the paper describe the usage of LLMs if it is an important, original, or non-standard component of the core methods in this research? Note that if the LLM is used only for writing, editing, or formatting purposes and does \emph{not} impact the core methodology, scientific rigor, or originality of the research, declaration is not required.
    \item[] Answer: \answerNA{} 
    \item[] Justification: -

\end{enumerate}

\end{document}